\documentclass{article}
\usepackage{iclr2027_conference,times}

\usepackage{amsmath,amssymb,amsthm}
\usepackage{mathtools}
\usepackage{bm}
\usepackage{graphicx}
\usepackage{placeins}
\usepackage{booktabs}
\usepackage{pifont}
\usepackage{xcolor}
\usepackage{pgfplots}
\pgfplotsset{compat=1.18}

\newcommand{\et}[2]{#1$^{\scriptscriptstyle\pm#2}$}
\newcommand{\ets}[2]{\underline{#1$^{\scriptscriptstyle\pm#2}$}}
\newcommand{\etb}[2]{\textbf{#1$^{\scriptscriptstyle\pm#2}$}}
\newcommand{\cmark}{\ding{51}}
\newcommand{\xmark}{\ding{55}}

\newcommand{\fitwidth}[2]{%
  \sbox0{#2}%
  \ifdim\wd0>#1 \resizebox{#1}{!}{\usebox0}%
  \else \usebox0\fi}

\usepackage[hidelinks]{hyperref}
\usepackage{url}

\newtheorem{theorem}{Theorem}[section]
\newtheorem{lemma}[theorem]{Lemma}
\newtheorem{proposition}[theorem]{Proposition}

\newtheorem{definition}[theorem]{Definition}
\newtheorem{assumption}[theorem]{Assumption}

\usepackage[capitalize]{cleveref}
\crefname{assumption}{assumption}{assumptions}
\Crefname{assumption}{Assumption}{Assumptions}

\newcommand{\bx}{\mathbf{x}}

\newcommand{\bs}{\mathbf{s}}

\newcommand{\ba}{\mathbf{a}}
\newcommand{\bc}{\mathbf{c}}
\newcommand{\bd}{\mathbf{d}}
\newcommand{\bq}{\mathbf{q}}
\newcommand{\bu}{\mathbf{u}}
\newcommand{\bv}{\mathbf{v}}
\newcommand{\bJ}{\mathbf{J}}
\newcommand{\bepsilon}{\boldsymbol{\epsilon}}
\newcommand{\bo}{\mathbf{o}}
\newcommand{\br}{\mathbf{r}}

\newcommand{\bA}{\mathbf{A}}
\newcommand{\bI}{\mathbf{I}}

\newcommand{\bz}{\mathbf{z}}
\newcommand{\by}{\mathbf{y}}

\title{FloodDiffusion 2: Efficient and Path Controllable Streaming Motion Generation}
\author{Yiyi Cai\textsuperscript{1}\thanks{work done in Alaya Lab}\quad
Yuhan Wu\textsuperscript{2}\quad
Kunhang Li\textsuperscript{1,2}\quad
Tu Fangyuan\textsuperscript{1,3}\quad
Xiangyue Zhang\textsuperscript{1,2}\\
\textbf{Qiaoge Li\textsuperscript{4}\quad
Zhixiang Wang\textsuperscript{1}\quad
Kaipeng Zhang\textsuperscript{1,}\thanks{Corresponding authors.}\quad
Haiyang Liu\textsuperscript{1,2,}\footnotemark[3]}\\
\textsuperscript{1}Alaya Lab\\
\textsuperscript{2}The University of Tokyo\\
\textsuperscript{3}Japan Advanced Institute of Science and Technology\\
\textsuperscript{4}China Mobile Communications Company Limited Research Institute
}

\iclrfinalcopy
\hypersetup{
    pdftitle={FloodDiffusion 2: Efficient and Path Controllable Streaming Motion Generation},
    pdfauthor={Yiyi Cai, Yuhan Wu, Kunhang Li, Tu Fangyuan, Xiangyue Zhang, Qiaoge Li, Zhixiang Wang, Kaipeng Zhang, Haiyang Liu}
}

\begin{document}
\setcounter{footnote}{1}
\maketitle
\lhead{Preprint.}

\begin{abstract}
We present FloodDiffusion 2 (FD2), an efficient and controllable framework that builds upon FloodDiffusion (FD1), a state-of-the-art streaming motion generation model. While FD1 produces plausible motion, it suffers from low efficiency and limited controllability, as its attention design requires repeated computation over the entire history, and it lacks precise trajectory control for real-world applications.
To address these limitations and improve generation quality, FD2 introduces three advances. First, \emph{Partial Attention} makes finalized history representations independent of the active window, enabling KV-cached inference and shared-history packing for efficient training. Second, we establish a necessary-and-sufficient Bregman criterion for regression losses to preserve diffusion's conditional-mean velocity field. This criterion guides an FK-induced quadratic loss that incorporates motion geometry without online FK evaluation. Third, FD2 introduces precise path conditioning to control the character's root trajectory while preserving natural body motion.
Experiments show that FD2 reduces training computation by 4.6$\times$ and accelerates denoising by 11.29$\times$, reaching 2.303 ms per update on long sequences. Alongside these efficiency gains, FD2 improves motion quality over FD1 and achieves state-of-the-art FID scores among streaming methods, with 0.048 on SEED and 0.053 on HumanML3D.
\end{abstract}

\begin{figure}[h]
    \centering
    \includegraphics[width=\linewidth]{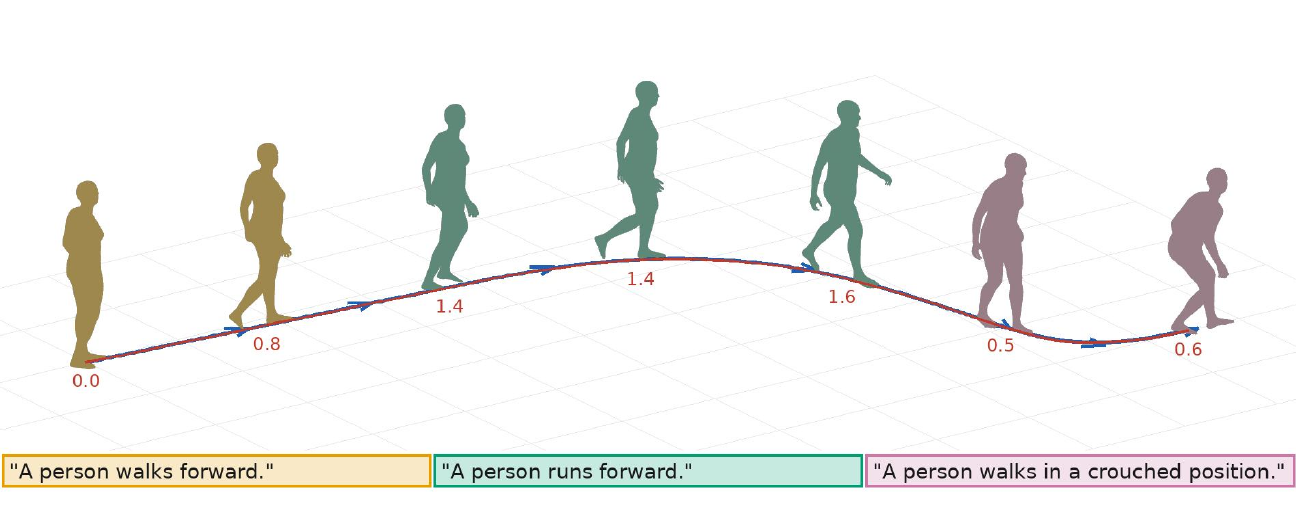}
    \caption{\textbf{FloodDiffusion 2} generates continuous human motion from time-varying text prompts while following a specified path. The character transitions smoothly from walking to running and then crouched walking. Blue and red curves denote the target and generated center-of-mass trajectories, respectively; numbers indicate the generated planar speed in m/s.}
    \label{fig:teaser}
\end{figure}

\section{Introduction}
\label{sec:introduction}

Streaming text-to-motion generation produces motion incrementally while accepting changing conditions~\citep{zhaodartcontrol,xiao2025motionstreamer,cai2025flooddiffusion,zhao2026ardy}. Interactive virtual characters require efficient generation and precise trajectory control alongside natural, text-aligned motion. FloodDiffusion (FD1), a state-of-the-art streaming motion generation model, realizes this process through a frame-staggered diffusion schedule: frames enter and complete denoising successively as the active region advances~\citep{cai2025flooddiffusion}.

FD1 has two practical limitations. First, its bidirectional attention allows finalized history to depend on the changing active window. History representations therefore change during denoising, preventing standard KV caching and requiring repeated computation during training and inference. Second, FD1 provides only text-based control, which cannot precisely specify even a character's walking direction. Real-world applications require the character to follow a specified route while performing different actions. We therefore introduce FloodDiffusion 2 (FD2) to improve efficiency and enable precise path control while preserving natural motion.

To improve efficiency, FD2 retains FD1's schedule and sampler and employs Partial Attention. History queries read their causal prefixes; active queries read the complete history and interact bidirectionally within the active window. This preserves access to the inputs of the streaming velocity field while making history representations independent of subsequent frames. FD2 uses this structure to support standard KV-cached inference and shared-history training. During inference, finalized history representations can be cached and reused as the active window advances. During training, one clean reference encoding is shared across multiple isolated denoising windows, reducing repeated history computation. 

To improve generation quality, we align supervision with human perception of motion. Errors in motion features, such as local joint rotations, do not directly reflect their effects on body geometry: rotation errors can accumulate along kinematic chains and produce noticeable positional errors at the hands and feet. A natural approach is to introduce a forward kinematics (FK) loss that penalizes errors in the reconstructed body geometry. However, directly applying a nonlinear FK loss can shift the regression optimum away from the conditional-mean velocity field required by diffusion. 

We therefore ask which losses preserve this velocity field while incorporating geometric supervision. Sampled velocities are regression targets, but the field to be learned is their conditional mean. Applying the classical mean-elicitation result~\citep{banerjee2005optimality,gneiting2011making}, we establish a necessary-and-sufficient criterion under the stated regularity and integrability conditions: the compatible losses are exactly target-first Bregman divergences generated by strictly convex potentials, up to a target-only term. Guided by this criterion, we construct an FK-induced quadratic from a dataset-averaged kinematic response, incorporating geometric sensitivity while preserving the conditional mean and avoiding online FK evaluation.

FD2 incorporates root-path conditioning to enable precise trajectory control. The model receives a frame-aligned root trajectory together with time-varying text prompts, allowing spatial movement and action semantics to be specified jointly. In our SEED representation, the root is a mass-weighted center-of-mass ground projection, with explicit root-relative pelvis translation and rotation. This separates the commanded trajectory from the body's articulation, reducing gait-cycle sway while preserving natural body motion. \Cref{fig:teaser} illustrates changing actions along a specified path.

Experiments on SEED and HumanML3D demonstrate improvements in both efficiency and motion quality. Shared-history packing reduces training denoiser computation by $4.6\times$. With 4,500 history frames, KV caching accelerates denoising by $11.29\times$ to 2.303~ms per update on an RTX 4090. Alongside these efficiency gains, FD2 improves HumanML3D FID from FD1's 0.057 to 0.053 and achieves an FID of 0.048 on SEED, yielding state-of-the-art FID scores among the compared streaming text-to-motion methods on both benchmarks.

Our contributions are:
\begin{itemize}
    \item \textbf{Partial Attention for efficient streaming diffusion}
    enables standard KV-cached inference and reduces training computation through shared-history multi-window supervision.
    \item \textbf{A necessary-and-sufficient characterization of diffusion-compatible losses}
    connects conditional-mean velocity regression to Bregman divergences and guides an FK-induced quadratic objective for geometric supervision.
    \item \textbf{Root-path conditioning}
    enables precise trajectory control while preserving natural root-relative body motion.
\end{itemize}

\section{Related Work}
\label{sec:relatedworks}

\paragraph{Motion generation and streaming.}
\label{sec:rw_df}
\label{sec:rw_motion}
Motion generation spans continuous diffusion~\citep{tevet2022human,zhang2024motiondiffuse}, autoregression~\citep{zhang2023generating}, retrieval augmentation~\citep{zhang2023remodiffuse}, and masked token modeling~\citep{guo2024momask}. MotionStreamer and ARDY generate incrementally~\citep{xiao2025motionstreamer,zhao2026ardy}. FD1 uses frame-staggered diffusion~\citep{cai2025flooddiffusion}, building on Diffusion Forcing's per-element noise levels~\citep{chen2024diffusion}. FD2 retains FD1's schedule and sampler.

\paragraph{Efficient streaming diffusion.}
\label{sec:rw_llm}
Transformer-XL reuses history states for autoregressive generation~\citep{dai2019transformerxl}. Rolling Forcing combines causal history with a bidirectional active window for KV caching~\citep{liu2025rolling}. It uses distribution-matching distillation from a pretrained teacher and, like Self Forcing~\citep{huang2025self}, trains on self-generated histories. FD2 employs Partial Attention for direct velocity supervision on motion data, sharing a clean reference encoding across isolated denoising windows.

\paragraph{Geometry-aware supervision.}
MDM supplements feature regression with position, velocity, and foot-contact losses~\citep{tevet2022human}. Such geometric supervision motivates our objective, but nonlinear FK losses can shift the conditional-mean velocity target. We apply classical mean-elicitation theory~\citep{banerjee2005optimality,gneiting2011making} to identify compatible regression losses and propose an FK-induced quadratic objective that preserves this target without online FK evaluation.

\paragraph{Spatially controlled motion generation.}
GMD and OmniControl support spatial joint and trajectory constraints alongside text~\citep{karunratanakul2023guided,xie2024omnicontrol}. Kimodo overwrites constrained features with clean values and supplies a control mask; its denoiser predicts root before body, using a smoothed root distinct from the pelvis~\citep{rempe2026kimodo}. ARDY overwrites constrained global root channels while generating latent body motion and accepts future path constraints beyond the active window~\citep{zhao2026ardy}. FD2 similarly holds three relative root channels clean and diffuses explicit body features within FD1's frame-staggered schedule. Its SEED root is a mass-weighted center-of-mass ground projection, with separate pelvis motion preserving root-relative articulation.

\section{Methods}
\label{sec:methods}

FD2 retains FD1's frame-staggered diffusion schedule and sampler~\citep{cai2025flooddiffusion}. We employ Partial Attention for efficient inference and training, characterize diffusion-compatible losses to design an FK-induced geometric objective, and incorporate root-path conditioning for precise trajectory control.

\subsection{Preliminaries}
\label{sec:preliminaries}
We use a 138-D rotation-direct state for SEED and native 263-D features for HumanML3D. For a clean standardized motion sequence $\bz$, independent noise $\bepsilon\sim\mathcal N(0,I)$, and frame-aligned conditions $\bc$, frame $k$ follows
\begin{equation}
    \alpha_t^k=\mathrm{clamp}(t-k/n_s,0,1),\qquad
    \beta_t^k=1-\alpha_t^k,
    \label{eq:tri_schedule}
\end{equation}
where $n_s$ controls the active-window width. The noisy state and sampled velocity are
\begin{equation}
    \bx_t^k=\alpha_t^k\bz^k+\beta_t^k\bepsilon^k,\qquad
    \bv_t^k=\dot\alpha_t^k\bz^k+\dot\beta_t^k\bepsilon^k.
    \label{eq:streaming_state}
\end{equation}
Away from schedule knots, finalized history $H=\{k:\alpha_t^k=1\}$ precedes the active window $A=\{k:0<\alpha_t^k<1\}$, where $\bv_t^k=\bz^k-\bepsilon^k$. Future frames remain pure noise; history and future have zero sampled velocity. The prescribed marginal velocity is
\begin{equation}
    u_t^A(\bx_t^{H\cup A},\bc^{H\cup A})
    =\mathbb E[\bv_t^A\mid\bx_t^{H\cup A},\bc^{H\cup A}].
    \label{eq:streaming_target}
\end{equation}
FD1's streaming-locality result makes this processed prefix sufficient under its causal-conditioning assumption. Appendix~\ref{sec:proof_partial_diffusion} states the assumption and exact interval indices.

\begin{figure}[htbp]
    \centering
    \includegraphics[width=\linewidth]{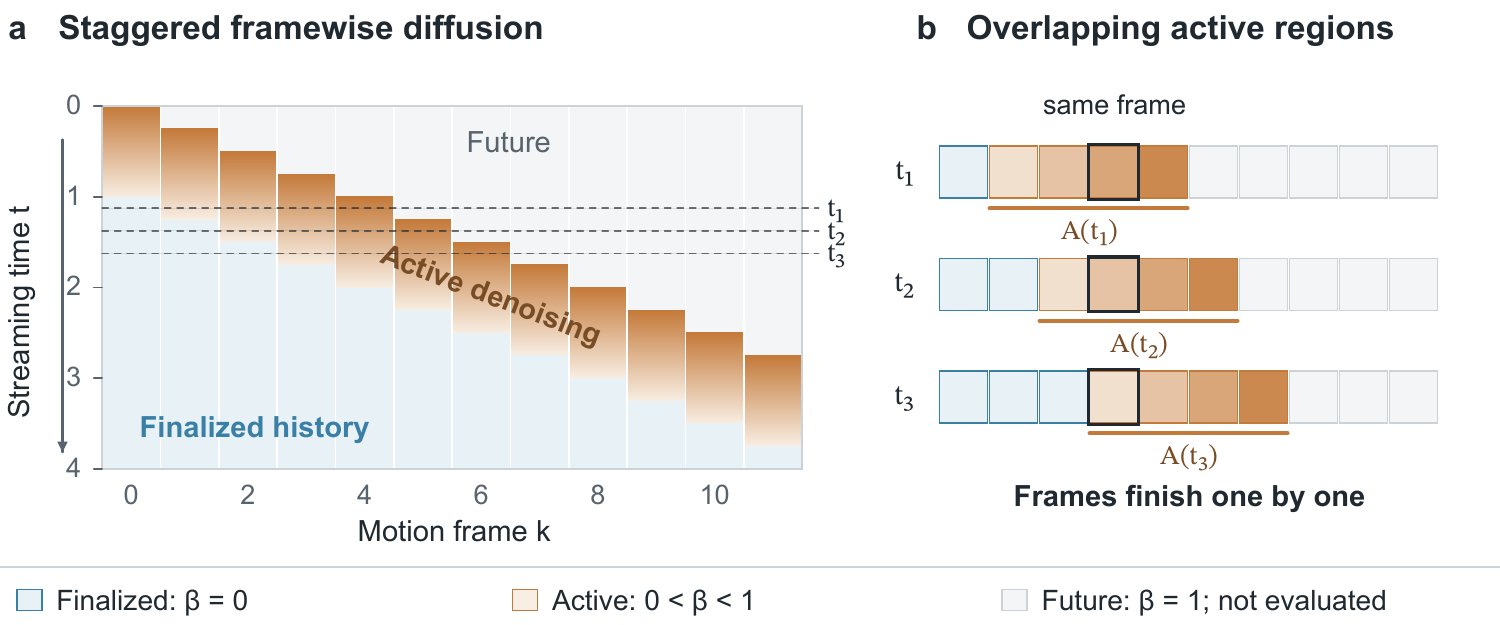}
    \caption{Frame-staggered diffusion schedule retained from FD1. (a) Time-shifted denoising ramps separate finalized history, the active window, and future noise. (b) Overlapping active windows advance as frames finish denoising one by one. The illustration uses $n_s=4$; $\beta$ denotes the noise amplitude.}
    \label{fig:staggered_schedule}
\end{figure}

\subsection{Partial Attention}
\label{sec:partial}
We employ Partial Attention to enable KV-cached inference and shared-history training. History queries are causal, while active queries read the complete processed prefix:
\begin{equation}
    \bA^{\mathrm p}_{j,k}=\begin{cases}
        1 & j\in H,\ k\le j,\\
        1 & j\in A,\ k\in H\cup A,\\
        0 & \text{otherwise.}
    \end{cases}
    \label{eq:partial_mask}
\end{equation}
Thus history can be encoded independently of later frames, while active outputs retain access to the information in \Cref{eq:streaming_target}. Appendix~\ref{sec:proof_partial_diffusion} establishes this decomposition and gives a data-law counterexample to fully causal active outputs.

\paragraph{KV-cache inference.}
With row-wise normalization and feed-forward operations, causal history states remain unchanged when their inputs, positions, noise metadata, and frame-aligned conditions are fixed. FD2 caches their layer-wise keys and values, re-encoding each newly finalized frame once under the causal mask before appending it. When prompts change, the conditions of finalized frames remain fixed, preserving their cached representations. Appendix~\ref{sec:proof_partial} proves prefix invariance.

\paragraph{Shared-history training.}
For each sequence, we sample one diffusion time from each of $K$ strata and pack the resulting noisy windows with one clean reference sequence $C$. A window $A_w$ starting at $s_w$ reads only $C_{<s_w}\cup A_w$; reference rows are causal. Noisy copies remain isolated even when their source intervals overlap. Under the conditions in Appendix~\ref{sec:proof_multiwindow}, their predictions match isolated calls, and the loss is averaged over valid target tokens. This shares history computation across $K$ training windows (\Cref{fig:partial_mask}).

\begin{figure}[htbp]
    \centering
    \includegraphics[width=\linewidth]{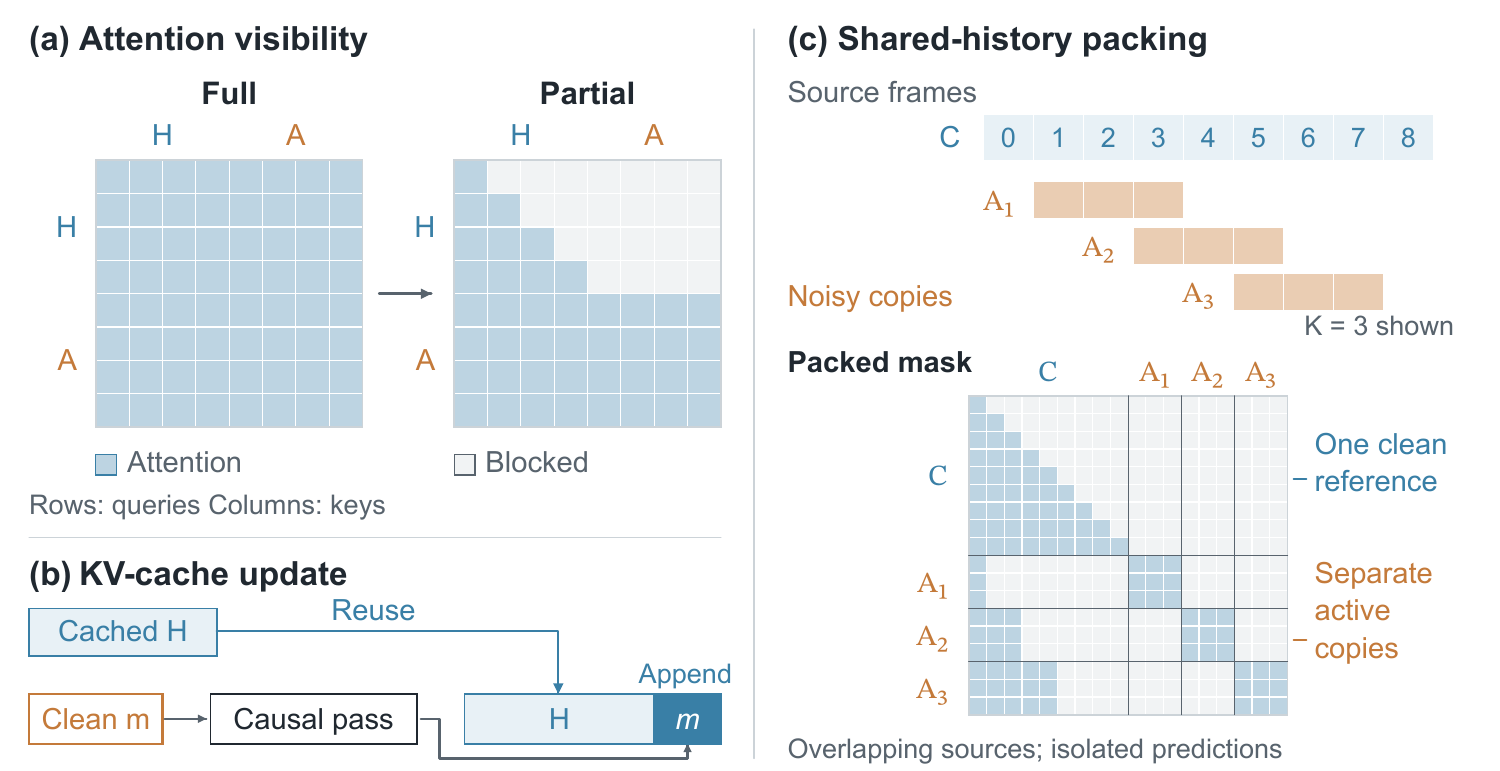}
    \caption{Partial Attention enables KV caching and shared-history training. (a) History queries read their causal prefixes, while active queries read the complete history and interact bidirectionally within the active window. (b) Newly finalized frames are re-encoded causally before their keys and values are appended to the cache. (c) Isolated noisy windows share one causal clean reference, reducing repeated history computation.}
    \label{fig:partial_mask}
\end{figure}

\subsection{Diffusion-Compatible Losses}
\label{sec:mesh_objective}
We seek geometric supervision that reflects the effects of motion-feature errors while preserving diffusion's conditional-mean velocity field (\Cref{eq:streaming_target}). Although training supervises sampled velocities, the loss must uniquely recover their conditional mean for every conditional target distribution.

Applying the classical mean-elicitation result~\citep{banerjee2005optimality,gneiting2011making}, we obtain a necessary-and-sufficient criterion: under the regularity and integrability conditions in Appendix~\ref{sec:app_loss_characterization}, these losses are exactly target-first Bregman divergences, up to a target-only term:
\begin{equation}
    \ell(y,a)=D_\varphi(y,a)+c(y),\qquad
    D_\varphi(y,a)=\varphi(y)-\varphi(a)-\nabla\varphi(a)^\top(y-a),
    \label{eq:mean_bregman_main}
\end{equation}
with strictly convex $\varphi$, sampled target $y$, and prediction $a$. Let $X$ denote the predictor inputs, $Y$ the sampled velocity, and $\mu=\mathbb E[Y\mid X]$. The conditional risk identity is
\begin{equation}
    \mathbb E[\ell(Y,a)-\ell(Y,\mu)\mid X]=D_\varphi(\mu,a).
    \label{eq:mean_risk_main}
\end{equation}
Strict convexity makes $\mu$ the unique optimum. Thus geometric weighting can change without shifting the required conditional mean. \Cref{fig:loss_characterization} summarizes this criterion and our construction.

\paragraph{Direct nonlinear FK supervision.}
Let $F_{\boldsymbol\eta}$ decode a clean standardized state to geometric output, with $\boldsymbol\eta$ supplying the relative representation's root context. At an active frame, the predicted velocity gives $\widehat{\bz}=\bx_t+\beta_t\widehat{\bv}$. Direct supervision with a fixed geometric metric $M\succeq0$ would minimize
\begin{equation}
    \ell_{\mathrm{direct}}(\bv,\widehat{\bv})
    =\left\|F_{\boldsymbol\eta}(\bz)-F_{\boldsymbol\eta}(\bx_t+\beta_t\widehat{\bv})\right\|_M^2.
    \label{eq:direct_fk_loss}
\end{equation}
For nonlinear $F_{\boldsymbol\eta}$, averaging geometric outputs generally differs from decoding the mean state. Consequently, this loss need not recover the conditional-mean velocity; Appendix~\ref{sec:app_loss_characterization} provides a counterexample. The controlled experiment in Appendix~\ref{sec:app_fk_sampling} further shows that nonlinear FK supervision can bias sampling toward high-density regions, reducing diversity.

\paragraph{An FK-induced quadratic.}
We retain geometric sensitivity through a linear response in standardized coordinates. For the 138-D state, $J_{\boldsymbol\eta,\bz}$ is the neutral-SMPL-H surface-decoder Jacobian: a small state perturbation $\boldsymbol\delta$ produces local geometric error $\boldsymbol\delta^\top J_{\boldsymbol\eta,\bz}^\top M J_{\boldsymbol\eta,\bz}\boldsymbol\delta$, where $M\succeq0$ weights vertices by surface area. For the 263-D features, $J_{\boldsymbol\eta,\bz}$ is a designed kinematic joint-response operator. We average and normalize the response Gram matrices using training data:
\begin{equation}
    \overline G=\mathbb E_{(\boldsymbol\eta,\bz)\sim p_{\mathrm{train}}}
    [J_{\boldsymbol\eta,\bz}^{\top}MJ_{\boldsymbol\eta,\bz}],\qquad
    W_{\mathrm{FK}}=\frac{d\,\overline G}{\operatorname{tr}(\overline G)}.
    \label{eq:mesh_pullback_main}
\end{equation}
Here $d$ is the number of predicted channels. The averaged matrix is independent of the current regression target; target-dependent per-example weights generally elicit a weighted mean. The clean-state residual is $\boldsymbol\delta=\beta_t\mathbf e$ for velocity residual $\mathbf e=\widehat{\bv}-\bv$. Omitting $\beta_t^2$ reweights time/frame contributions while preserving their conditional-mean optima. We use the fixed channel metric
\begin{equation}
    W_\gamma=\frac{I_d+\gamma W_{\mathrm{FK}}}{1+\gamma},\qquad
    \mathcal L_\gamma=\frac1d\mathbf e^\top W_\gamma\mathbf e,\qquad \gamma\ge0.
    \label{eq:mesh_quadratic}
\end{equation}
The identity term makes $W_\gamma$ positive definite, and trace normalization keeps its trace at $d$. The potential $\varphi_\gamma(\bu)=\bu^\top W_\gamma\bu/d$ generates this quadratic as a target-first Bregman divergence, so \Cref{eq:mean_risk_main} guarantees the same conditional-mean velocity optimum. The matrix is estimated once and held fixed; training requires no online FK or Jacobians. FD2-Path restricts the response to its predicted channels before averaging and normalization. Appendices~\ref{sec:app_mesh_quadratic} and~\ref{sec:app_representation} give the complete construction and decoder.

\begin{figure}[!t]
    \centering
    \includegraphics[width=\linewidth]{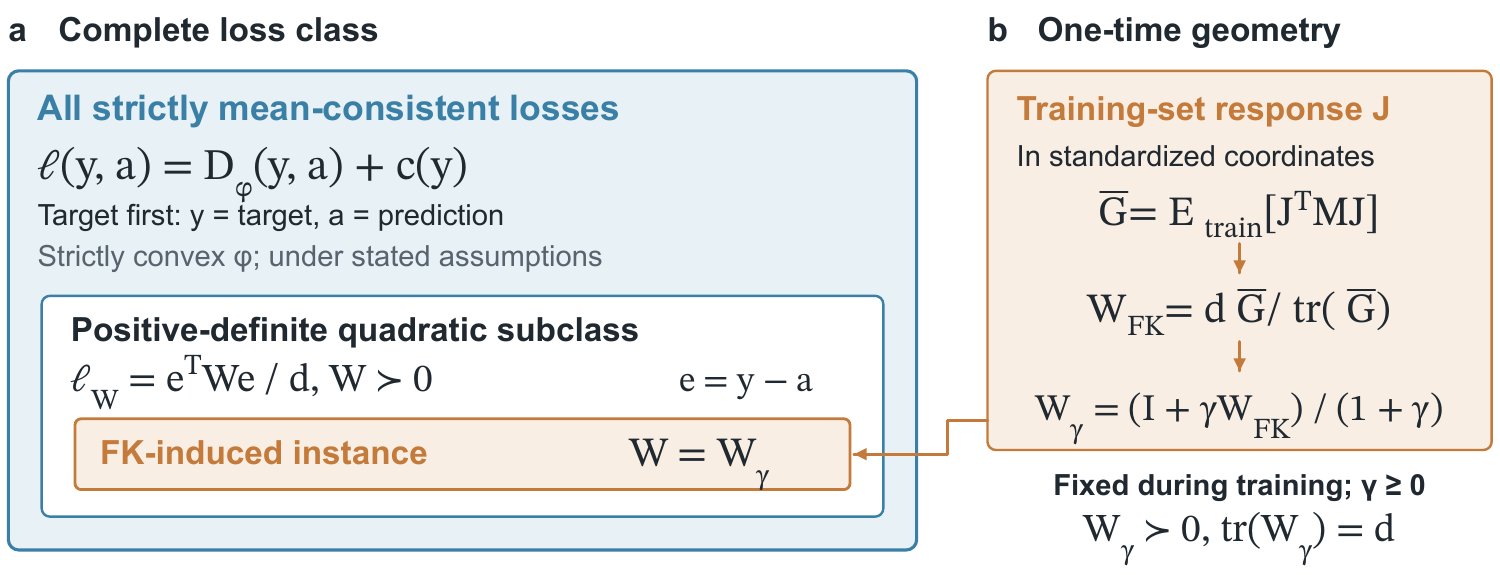}
    \caption{Diffusion-compatible losses and the FK-induced quadratic objective. (a) Under the stated regularity and integrability conditions, target-first Bregman divergences with strictly convex potentials, up to a target-only term, are necessary and sufficient for uniquely recovering the conditional mean. (b) Kinematic response Gram matrices estimated from training data are averaged and normalized, then combined with an identity term to obtain a fixed positive-definite metric, avoiding online FK evaluations.}
    \label{fig:loss_characterization}
\end{figure}

\subsection{Root-Path Conditioning}
\label{sec:position_control}
FD2-Path, the path-conditioned variant of FD2, receives a frame-aligned planar root trajectory $(x_k,z_k,\psi_k)$ together with text. We encode the trajectory into three relative root channels $\bc_k$: heading increments and heading-local displacement. These channels remain clean and are concatenated with the noisy remaining state $\bx_t^k$:
\begin{equation}
    \widehat{\bv}_t^k=f_\theta\!\left(
    [\mathcal N_c(\bc_k)\,\Vert\,\bx_t^k],t,E_{\mathrm{text}}(\by);\bA^{\mathrm p}\right).
    \label{eq:position_condition}
\end{equation}
Here $\mathcal N_c$ standardizes the root conditions per channel and $\by$ denotes the text prompt. Only the remaining $d=135$ (138-D) or $260$ (263-D) channels are diffused and supervised. Sampling combines the prescribed root channels with the generated body state.

\paragraph{Center-of-mass root and relative body motion.}
For the 138-D SEED state, we separate the planar root from the articulated pelvis. Given the 22 body joints $\bJ_{k,i}$, the root is the ground projection of their fixed mass-weighted average:
\begin{equation}
    \bq_k=\Pi_{\mathrm{ground}}\!\left(\sum_{i=0}^{21}w_i\bJ_{k,i}\right),
    \qquad \sum_i w_i=1.
    \label{eq:com_root_main}
\end{equation}
The weights follow biomechanical segment parameters~\citep{deleva1996adjustments}. To compare root definitions, we evaluate their smoothness on the HumanML3D training split: this root has 20.3\% lower planar jerk and 9.6\% lower curvature than a pelvis projection (Appendix~\Cref{tab:root_smoothness}). The pelvis retains its translation and 3D rotation relative to the root, allowing natural body motion around the commanded trajectory. Together with the 21 parent-local body rotations, these channels decode directly to neutral SMPL-H. The supplied root channels remain fixed during sampling. Tracking error is measured from the generated body's center-of-mass projection in \Cref{eq:com_root_main}, which depends jointly on the predicted pelvis offset and articulation. Appendix~\ref{sec:app_representation} gives the weights, heading definition, and encoding conventions.

\section{Experiments}
\label{sec:experiments}
\label{sec:results}

\subsection{Experimental Setup}
\paragraph{\normalfont Datasets and training.}
HumanML3D uses the standard split and evaluator~\citep{guo2022generating}, with FD1's preprocessing and duration handling. SEED~\citep{bonesstudio2026seed} uses the public Kimodo training partition and category-held-out Content test split with the v1.1 evaluator~\citep{kimodo2026benchmark}. SEED R@3/FID use TMR-SOMA-RP-v1 embeddings; Skate measures contact-time foot velocity (cm/s). Appendix~\ref{sec:app_seed_eval} details conversion and evaluation.

Our 114M-parameter transformer uses frozen UMT5-XXL~\citep{chung2023unimax}, $n_s=30$, $K=8$, and $\gamma=1$ ($0$ without FK). On HumanML3D, FD2 uses the 60K checkpoint with CFG 4, and FD2-Path uses the 55K checkpoint with CFG 3. Both SEED models use 300K checkpoints with CFG 2. Appendix~\ref{sec:app_training_config} details the configuration.

\subsection{Attention and Efficiency}
\label{sec:results_attention}
\paragraph{\normalfont Attention masks.}
Partial Attention retains comparable FID to Full attention (0.058 vs. 0.057) and improves R@3 and MM-Dist; fully causal attention gives FID 0.096 (\Cref{tab:attention_ablation}).

\begin{table}[htbp]
\centering
\small
\caption{Attention-mask ablation on HumanML3D with the same 263-D configuration, checkpoint budget, and $K=1$. Diversity closer to real motion (9.503) is better.}
\label{tab:attention_ablation}
\setlength{\tabcolsep}{5pt}
\begin{tabular}{l c c c c}
\toprule
Attention & R@3$\uparrow$ & FID$\downarrow$ & MM-Dist$\downarrow$ & Div.$\rightarrow$\\
\midrule
Full    & $0.810$ & $\mathbf{0.057}$ & $2.887$ & $9.579$\\
Causal  & $0.793$ & $0.096$ & $3.005$ & $9.369$\\
Partial & $\mathbf{0.814}$ & $0.058$ & $\mathbf{2.839}$ & $\mathbf{9.438}$\\
\bottomrule
\end{tabular}
\end{table}

\paragraph{\normalfont Multi-window training.}
\label{sec:results_multiwindow}
At 512 windows per step, $K=8$ packing reduces training denoiser FLOPs by $4.6\times$ and GPU count from four to one (\Cref{tab:train_eff}). One GPU also supports 1,024 windows. FD2 retains motion quality, which improves with FK supervision (\Cref{tab:quantitative_eval}).

\begin{table*}[htbp]
\centering
\small
\caption{Multi-window training with 300 frames and 16 text tokens (rounded mean caption length). TFLOPs cover denoiser forward/backward; memory sums peak usage across H200 GPUs.}
\label{tab:train_eff}
\setlength{\tabcolsep}{5pt}
\fitwidth{\textwidth}{%
\begin{tabular}{c c c c c c}
\toprule
$K$ & GPUs & Batch/GPU & Windows/step & TFLOPs$\downarrow$
& Total peak GPU GiB$\downarrow$\\
\midrule
$1$  & $4$ & $128$ & $512$   & $90.666$ & $164.27$\\
$4$  & $1$ & $128$ & $512$   & $29.827$ & $54.27$\\
$8$  & $1$ & $64$  & $512$   & $19.846$ & $36.61$\\
$16$ & $1$ & $32$  & $512$   & $\mathbf{14.903}$ & $\mathbf{27.97}$\\
\midrule
$8$  & $1$ & $128$ & $1{,}024$ & $39.691$ & $72.37$\\
\bottomrule
\end{tabular}
}
\end{table*}

\paragraph{\normalfont Cached inference.}
\label{sec:results_kvcache}
At 4,500 history frames, caching reduces backbone latency from 26.002 to 2.303\,ms per update ($11.29\times$) and computation by $137\times$ (\Cref{tab:kv_cache}).

\begin{table}[!htbp]
\centering
\small
\caption{KV-cache inference on RTX 4090: batch 1, 30 active frames, 16 cached text tokens, CUDA Graph, PyTorch SDPA, no CFG. FLOPs and latency include the complete backbone update, causal re-encoding, and K/V writes for the newly finalized frame. Cache counts bf16 history K/V.}
\label{tab:kv_cache}
\setlength{\tabcolsep}{4pt}
\begin{tabular}{c c c c c c c}
\toprule
& \multicolumn{2}{c}{Full, uncached} & \multicolumn{2}{c}{Partial, KV cached} & Cache & Speedup \\
\cmidrule(lr){2-3}\cmidrule(lr){4-5}
History & GFLOPs$\downarrow$ & ms$\downarrow$ & GFLOPs$\downarrow$ & ms$\downarrow$ & MiB & $\times$ \\
\midrule
$0$ & $6.26$ & $\mathbf{1.810}$ & $6.26$ & $1.831$ & $0$ & $0.99$ \\
$90$ & $23.26$ & $2.107$ & $\mathbf{6.53}$ & $\mathbf{1.953}$ & $2.8$ & $1.08$ \\
$450$ & $96.61$ & $3.066$ & $\mathbf{6.90}$ & $\mathbf{1.962}$ & $14.1$ & $1.56$ \\
$4500$ & $1506.95$ & $26.002$ & $\mathbf{11.01}$ & $\mathbf{2.303}$ & $140.6$ & $\mathbf{11.29}$ \\
\bottomrule
\end{tabular}
\end{table}

\subsection{Motion Quality and Path Control}
\label{sec:results_quality}
\paragraph{\normalfont HumanML3D.}
FK lowers FD2's FID from 0.058 (FD1: 0.057) to 0.053, the best among compared streaming text-only methods (\Cref{tab:quantitative_eval}). FD2-Path leads the compared path-conditioned methods in R@3, FID, MM-Dist, and 3D pelvis error. Root channels are fixed by conditioning.

\begin{table}[!htbp]
\centering
\small
\caption{HumanML3D text-only and path-conditioned generation: means with 95\% confidence intervals; bold/underline denote best/second-best per condition. Text-only baselines and adapted PRIMAL follow FD1; path results are our evaluations. Root measures mean planar control-point error; 3D, mean pelvis error against ground truth (both cm). Skate: fraction of frames with foot sliding. Control targets: Appendix~\ref{sec:app_path_conventions}.}
\label{tab:quantitative_eval}
\label{tab:humanml_path}
\setlength{\tabcolsep}{2pt}
\fitwidth{\textwidth}{%
\begin{tabular}{l c c c c c c c c}
\toprule
Method & Stream & R@3$\uparrow$ & FID$\downarrow$ & MM-Dist$\downarrow$ & Div.$\rightarrow$ & Root$\downarrow$ & 3D$\downarrow$ & Skate$\downarrow$ \\
\midrule
\multicolumn{9}{l}{\textbf{Text-only}}\\
Real motion & -- & \et{0.797}{.002} & \et{0.002}{.000} & \et{2.974}{.008} & \et{9.503}{.065} & -- & -- & -- \\
TM2T~\citep{guo2022tm2t} & \xmark & \et{0.729}{.002} & \et{1.501}{.017} & \et{3.467}{.011} & \et{8.589}{.076} & -- & -- & -- \\
T2M~\citep{guo2022generating} & \xmark & \et{0.736}{.002} & \et{1.087}{.021} & \et{3.347}{.008} & \et{9.175}{.083} & -- & -- & -- \\
MDM~\citep{tevet2022human} & \xmark & \et{0.611}{.007} & \et{0.544}{.044} & \et{5.566}{.027} & \ets{9.559}{.086} & -- & -- & -- \\
MLD~\citep{chen2023executing} & \xmark & \et{0.772}{.002} & \et{0.473}{.013} & \et{3.196}{.010} & \et{9.724}{.082} & -- & -- & -- \\
MotionDiffuse~\citep{zhang2024motiondiffuse} & \xmark & \et{0.782}{.001} & \et{0.630}{.001} & \et{3.113}{.001} & \et{9.410}{.049} & -- & -- & -- \\
T2M-GPT~\citep{zhang2023generating} & \xmark & \et{0.775}{.002} & \et{0.141}{.005} & \et{3.121}{.009} & \et{9.722}{.082} & -- & -- & -- \\
ReMoDiffuse~\citep{zhang2023remodiffuse} & \xmark & \et{0.795}{.004} & \et{0.103}{.004} & \et{2.974}{.016} & \et{9.018}{.075} & -- & -- & -- \\
MoMask~\citep{guo2024momask} & \xmark & \et{0.807}{.002} & \etb{0.045}{.002} & \et{2.958}{.008} & \et{9.677}{.032} & -- & -- & -- \\
\midrule
PRIMAL~\citep{zhang2025primal} & \cmark & \et{0.780}{.002} & \et{0.511}{.006} & \et{3.120}{.019} & \etb{9.520}{.068} & -- & -- & -- \\
MotionStreamer~\citep{xiao2025motionstreamer} & \cmark & \et{0.802}{.003} & \et{0.092}{.003} & \et{2.909}{.015} & \et{9.722}{.037} & -- & -- & -- \\
FloodDiffusion~\citep{cai2025flooddiffusion} & \cmark & \ets{0.810}{.003} & \et{0.057}{.002} & \ets{2.887}{.007} & \et{9.579}{.062} & -- & -- & -- \\
FD2 w/o FK (ours) & \cmark & \et{0.807}{.003} & \et{0.058}{.004} & \et{2.903}{.011} & \et{9.405}{.043} & -- & -- & -- \\
FD2 (ours) & \cmark & \etb{0.815}{.003} & \ets{0.053}{.004} & \etb{2.832}{.012} & \et{9.425}{.065} & -- & -- & -- \\
\midrule
\multicolumn{9}{l}{\textbf{Text + Path}}\\
GMD~\citep{karunratanakul2023guided} & \xmark & \et{0.777}{.003} & \et{0.133}{.012} & \et{3.129}{.020} & \ets{9.445}{.078} & \ets{0.003}{.000} & \et{7.19}{.13} & \et{0.144}{.001}\\
OmniControl~\citep{xie2024omnicontrol} & \xmark & \ets{0.804}{.001} & \ets{0.100}{.004} & \ets{2.928}{.012} & \et{9.688}{.007} & \et{5.48}{.09} & \ets{6.47}{.07} & \etb{0.071}{.000}\\
DART~\citep{zhaodartcontrol} & \cmark & \et{0.722}{.003} & \et{0.854}{.008} & \et{3.513}{.019} & \etb{9.498}{.014} & \et{7.26}{.05} & \et{12.40}{.15} & \et{0.130}{.001}\\
FD2-Path (ours) & \cmark & \etb{0.816}{.003} & \etb{0.050}{.004} & \etb{2.829}{.010} & \et{9.312}{.069} & \etb{0.000}{.000} & \etb{5.23}{.08} & \ets{0.082}{.000}\\
\bottomrule
\end{tabular}
}
\end{table}

\paragraph{\normalfont SEED.}
\label{sec:results_seed}
FD2 has the lowest streaming text-only FID and lower foot skate than ARDY/Kimodo, although its R@3 remains below ARDY (\Cref{tab:seed_results}). FK supervision improves all three text-only metrics and reduces FD2-Path's tracking error from 0.18 to 0.11\,cm and foot skate from 4.07 to 3.24\,cm/s.

\begin{table}[!htb]
\centering
\small
\caption{SEED Content Overview with the v1.1 evaluator. Ground truth and Kimodo text-only R@3/FID are official~\citep{kimodo2026benchmark}. Generated-motion Skate uses shared contact thresholds (parentheses: official predicted contacts). All path models receive the same 917 dense paths. Position pools conditioned frames: generated-body CoM for FD2-Path, pelvis for Kimodo/ARDY. Bold: best streaming result per condition. $^{\dagger}$ARDY uses RP01 training; FD2/Kimodo use the public SEED partition. Motion Data by Bones Studio.}
\label{tab:seed_results}
\setlength{\tabcolsep}{3pt}
\fitwidth{\linewidth}{%
\begin{tabular}{l c c c c c c}
\toprule
& & & & & \multicolumn{2}{c}{Position (cm)$\downarrow$}\\
\cmidrule(lr){6-7}
Method & Stream & R@3$\uparrow$ & FID$\downarrow$ & Skate$\downarrow$ & Mean & P95\\
\midrule
\multicolumn{7}{l}{\textbf{Text-only}}\\
\addlinespace[2pt]
Ground truth & -- & 89.09 & 0.000 & 1.85 & -- & --\\
Kimodo-SOMA-SEED-v1.1~\citep{rempe2026kimodo} & \xmark & 81.13 & 0.035 & 3.34 (3.70) & -- & --\\
\cmidrule(lr){1-7}
ARDY~\citep{zhao2026ardy}$^{\dagger}$ & \cmark & \textbf{76.34} & 0.055 & 3.17 (4.43) & -- & --\\
FloodDiffusion~2 w/o FK loss (ours) & \cmark & 70.88 & 0.053 & 3.84 & -- & --\\
FloodDiffusion~2 (ours) & \cmark & 73.12 & \textbf{0.048} & \textbf{3.08} & -- & --\\
\midrule
\multicolumn{7}{l}{\textbf{Text + Path}}\\
\addlinespace[2pt]
Kimodo-SOMA-SEED-v1.1~\citep{rempe2026kimodo} & \xmark & 78.84 & 0.031 & 4.14 (5.14) & 2.37 & 7.44\\
\cmidrule(lr){1-7}
ARDY~\citep{zhao2026ardy}$^{\dagger}$ & \cmark & 75.03 & 0.052 & 3.63 (5.11) & 1.60 & 3.91\\
FloodDiffusion~2-Path w/o FK loss (ours) & \cmark & 75.07 & 0.036 & 4.07 & 0.18 & 0.55\\
FloodDiffusion~2-Path (ours) & \cmark & \textbf{76.01} & \textbf{0.034} & \textbf{3.24} & \textbf{0.11} & \textbf{0.32}\\
\bottomrule
\end{tabular}}
\end{table}

\FloatBarrier
\subsection{Qualitative Control}
\label{sec:results_position}
\Cref{fig:prompt_switch,fig:speed_control} illustrate action and gait changes under text and path conditioning.

\begin{figure}[!htbp]
    \centering
    \includegraphics[width=0.80\linewidth]{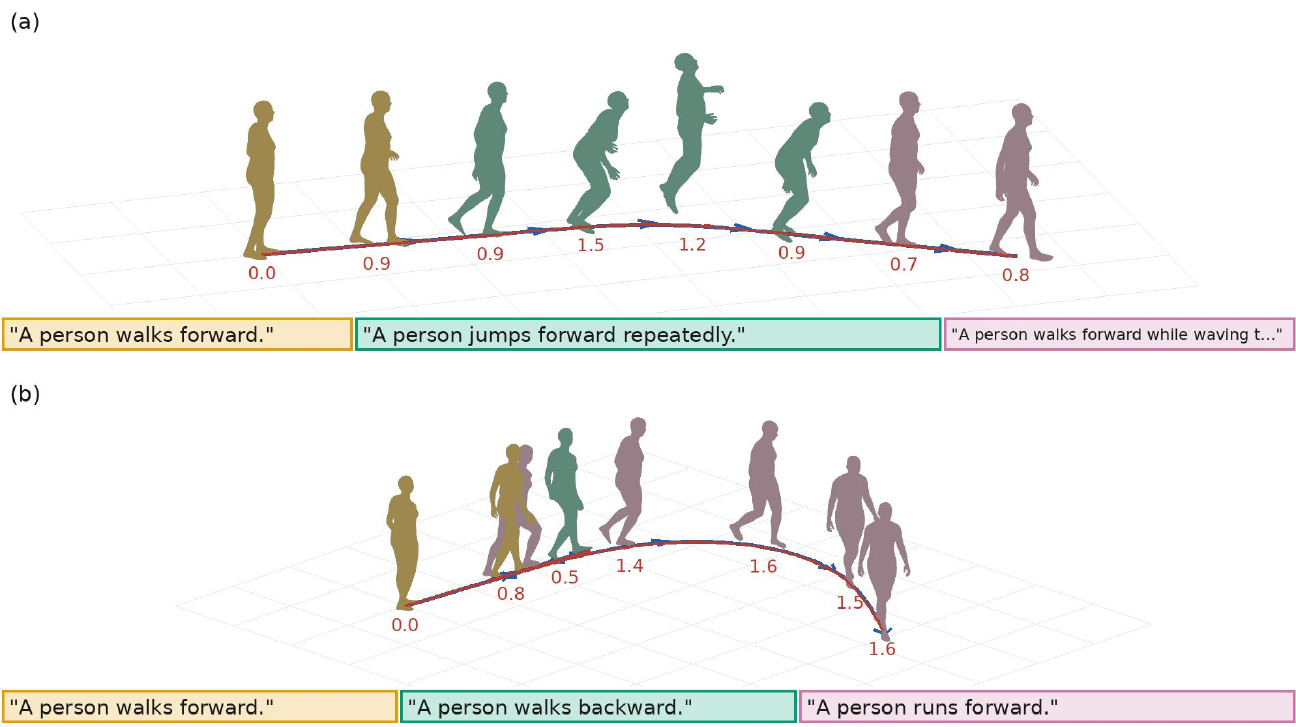}
    \caption{Prompt switching along a commanded path. Body colors indicate the active prompt. Blue is the commanded center-of-mass path, red the generated center-of-mass track; numbers show planar speed in m/s. The examples include repeated jumping and backward walking.}
    \label{fig:prompt_switch}
\end{figure}

\begin{figure}[!htbp]
    \centering
    \includegraphics[width=0.80\linewidth]{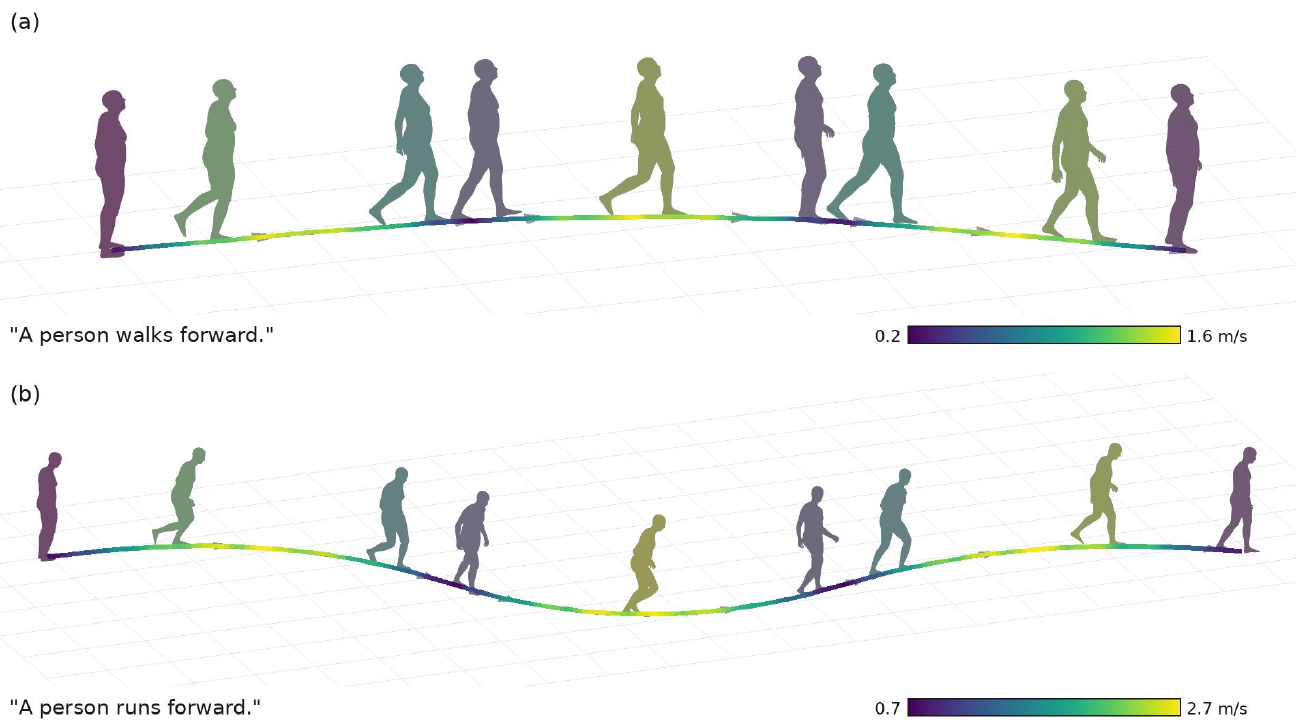}
    \caption{Speed control under fixed text prompts. The walking (a) and running (b) sequences follow varying commanded speed profiles. Body and trajectory colors indicate generated planar speed, illustrating gait adaptation along the commanded path while the action prompt remains unchanged.}
    \label{fig:speed_control}
\end{figure}

\FloatBarrier
\section{Limitations and Conclusions}
\label{sec:conclusions}

\textbf{Conclusion.} FD2 improves FD1 through Partial Attention for efficient training and inference, an FK-induced quadratic guided by the necessary-and-sufficient Bregman criterion, and root-path conditioning. Experiments on HumanML3D and SEED demonstrate improved motion quality, efficiency, and controllability.

\textbf{Limitations.} Finalized errors persist in history, and generation beyond the training horizon requires extrapolation. Cache memory and per-update attention cost grow linearly with history length. Our fixed, dataset-averaged FK quadratic leaves room for more adaptive mean-consistent geometric weighting. Motion and transition diversity are limited by training-data coverage, particularly for rare actions and unfamiliar combinations.

\subsection*{AI use statement}
Generative AI tools assisted with manuscript and citation checks,
\LaTeX{} editing, schematic artwork, and benchmarking code. The authors reviewed the final manuscript and take responsibility for its content.

{
    \small
    \bibliographystyle{iclr2027_conference}
    \bibliography{main}

@String(CVPR= {IEEE Conf. Comput. Vis. Pattern Recog.})

@String(ICCV= {Int. Conf. Comput. Vis.})

@String(ECCV= {Eur. Conf. Comput. Vis.})

@inproceedings{cai2025flooddiffusion,
  title = {{{FloodDiffusion: Tailored Diffusion Forcing for Streaming Motion Generation}}},
  author = {Cai, Yiyi and Wu, Yuhan and Li, Kunhang and Zhou, You and Zheng, Bo and Liu, Haiyang},
  booktitle = {Proceedings of the IEEE/CVF Conference on Computer Vision and Pattern Recognition},
  year = {2026},
  pages = {2295--2304},
  url = {https://openaccess.thecvf.com/content/CVPR2026/html/Cai_FloodDiffusion_Tailored_Diffusion_Forcing_for_Streaming_Motion_Generation_CVPR_2026_paper.html}
}

@article{zhao2026ardy,
  title = {{{ARDY: Autoregressive Diffusion with Hybrid Representation for Interactive Human Motion Generation}}},
  author = {Zhao, Kaifeng and Petrovich, Mathis and Zhang, Haotian and Wang, Tingwu and Tang, Siyu and Rempe, Davis},
  journal = {ACM Transactions on Graphics},
  year = {2026},
  volume = {45},
  number = {4},
  pages = {86:1--86:14},
  doi = {10.1145/3811284},
  articleno = {86},
  numpages = {14}
}

@article{rempe2026kimodo,
  title = {{{Kimodo: Scaling Controllable Human Motion Generation}}},
  author = {Rempe, Davis and Petrovich, Mathis and Yuan, Ye and Zhang, Haotian and Peng, Xue Bin and Jiang, Yifeng and Wang, Tingwu and Iqbal, Umar and Minor, David and de Ruyter, Michael and Li, Jiefeng and Tessler, Chen and Lim, Edy and Jeong, Eugene and Wu, Sam and Hassani, Ehsan and Huang, Michael and Yu, Jin-Bey and Chung, Chaeyeon and Song, Lina and Dionne, Olivier and Kautz, Jan and Yuen, Simon and Fidler, Sanja},
  journal = {arXiv preprint arXiv:2603.15546},
  year = {2026},
  doi = {10.48550/arXiv.2603.15546},
  url = {https://arxiv.org/abs/2603.15546}
}

@misc{bonesstudio2026seed,
  title = {{{BONES-SEED: Skeletal Everyday Embodiment Dataset}}},
  author = {{Bones Studio}},
  howpublished = {Hugging Face dataset},
  year = {2026},
  url = {https://huggingface.co/datasets/bones-studio/seed},
  note = {Motion Data by Bones Studio; accessed 2026-09-05}
}

@misc{kimodo2026benchmark,
  title = {{{Kimodo Results}}},
  author = {{NVIDIA}},
  howpublished = {Kimodo documentation},
  year = {2026},
  url = {https://research.nvidia.com/labs/sil/projects/kimodo/docs/benchmark/results.html},
  note = {Results for Kimodo-SOMA-SEED-v1.1 and Kimodo-SOMA-RP-v1.1; accessed 2026-09-05}
}

@inproceedings{karunratanakul2023guided,
  title = {{{Guided Motion Diffusion for Controllable Human Motion Synthesis}}},
  author = {Karunratanakul, Korrawe and Preechakul, Konpat and Suwajanakorn, Supasorn and Tang, Siyu},
  booktitle = {Proceedings of the IEEE/CVF International Conference on Computer Vision},
  year = {2023},
  pages = {2151--2162},
  doi = {10.1109/ICCV51070.2023.00205},
}

@inproceedings{xie2024omnicontrol,
  title = {{{OmniControl: Control Any Joint at Any Time for Human Motion Generation}}},
  author = {Xie, Yiming and Jampani, Varun and Zhong, Lei and Sun, Deqing and Jiang, Huaizu},
  booktitle = {International Conference on Learning Representations},
  year = {2024},
  volume = {2024},
  pages = {28176--28194},
  url = {https://proceedings.iclr.cc/paper_files/paper/2024/hash/7827290f07f63485b849b66cfa3e5dd0-Abstract-Conference.html}
}

@inproceedings{dai2019transformerxl,
  title = {{Transformer-{XL}: Attentive Language Models beyond a Fixed-Length Context}},
  author = {Dai, Zihang and Yang, Zhilin and Yang, Yiming and Carbonell, Jaime and Le, Quoc V. and Salakhutdinov, Ruslan},
  booktitle = {Proceedings of the 57th Annual Meeting of the Association for Computational Linguistics},
  year = {2019},
  pages = {2978--2988},
  doi = {10.18653/v1/P19-1285},
  url = {https://aclanthology.org/P19-1285/},
  publisher = {Association for Computational Linguistics}
}

@inproceedings{mahmood2019amass,
  title = {{{AMASS: Archive of Motion Capture As Surface Shapes}}},
  author = {Mahmood, Naureen and Ghorbani, Nima and Troje, Nikolaus F. and Pons-Moll, Gerard and Black, Michael J.},
  booktitle = {2019 IEEE/CVF International Conference on Computer Vision (ICCV)},
  pages = {5441--5450},
  year = {2019},
  doi = {10.1109/ICCV.2019.00554},
}

@article{romero2017embodied,
  title = {{{Embodied Hands: Modeling and Capturing Hands and Bodies Together}}},
  author = {Romero, Javier and Tzionas, Dimitrios and Black, Michael J.},
  journal = {ACM Transactions on Graphics},
  volume = {36},
  number = {6},
  pages = {245:1--245:17},
  articleno = {245},
  year = {2017},
  doi = {10.1145/3130800.3130883},
}

@inproceedings{chung2023unimax,
  title = {{{UniMax}: Fairer and More Effective Language Sampling for Large-Scale Multilingual Pretraining}},
  author = {Chung, Hyung Won and Constant, Noah and Garcia, Xavier and Roberts, Adam and Tay, Yi and Narang, Sharan and Firat, Orhan},
  booktitle = {International Conference on Learning Representations},
  year = {2023},
  url = {https://openreview.net/forum?id=kXwdL1cWOAi}
}

@article{deleva1996adjustments,
  title = {{Adjustments to {Zatsiorsky--Seluyanov}'s Segment Inertia Parameters}},
  author = {de Leva, Paolo},
  journal = {Journal of Biomechanics},
  year = {1996},
  volume = {29},
  number = {9},
  pages = {1223--1230},
  doi = {10.1016/0021-9290(95)00178-6},
}

@inproceedings{tevet2022human,
  title = {{{Human Motion Diffusion Model}}},
  author = {Tevet, Guy and Raab, Sigal and Gordon, Brian and Shafir, Yonatan and Cohen-Or, Daniel and Bermano, Amit H.},
  booktitle = {International Conference on Learning Representations},
  year = {2023},
  url = {https://openreview.net/forum?id=SJ1kSyO2jwu}
}

@article{zhang2024motiondiffuse,
  title = {{{MotionDiffuse: Text-Driven Human Motion Generation With Diffusion Model}}},
  author = {Zhang, Mingyuan and Cai, Zhongang and Pan, Liang and Hong, Fangzhou and Guo, Xinying and Yang, Lei and Liu, Ziwei},
  journal = {IEEE Transactions on Pattern Analysis and Machine Intelligence},
  year = {2024},
  volume = {46},
  number = {6},
  pages = {4115--4128},
  doi = {10.1109/TPAMI.2024.3355414},
  publisher = {IEEE}
}

@inproceedings{zhang2023generating,
  title = {{{Generating Human Motion from Textual Descriptions with Discrete Representations}}},
  author = {Zhang, Jianrong and Zhang, Yangsong and Cun, Xiaodong and Zhang, Yong and Zhao, Hongwei and Lu, Hongtao and Shen, Xi and Shan, Ying},
  booktitle = {Proceedings of the IEEE/CVF Conference on Computer Vision and Pattern Recognition},
  year = {2023},
  pages = {14730--14740},
  doi = {10.1109/CVPR52729.2023.01415},
}

@inproceedings{guo2022tm2t,
  title = {{{TM2T: Stochastic and Tokenized Modeling for the Reciprocal Generation of 3D Human Motions and Texts}}},
  author = {Guo, Chuan and Zuo, Xinxin and Wang, Sen and Cheng, Li},
  booktitle = {Computer Vision -- ECCV 2022},
  year = {2022},
  volume = {13695},
  pages = {580--597},
  doi = {10.1007/978-3-031-19833-5_34},
  url = {https://link.springer.com/chapter/10.1007/978-3-031-19833-5_34},
  series = {Lecture Notes in Computer Science},
  publisher = {Springer Nature Switzerland},
  address = {Cham}
}

@inproceedings{zhaodartcontrol,
  title = {{{DartControl: A Diffusion-Based Autoregressive Motion Model for Real-Time Text-Driven Motion Control}}},
  author = {Zhao, Kaifeng and Li, Gen and Tang, Siyu},
  booktitle = {International Conference on Learning Representations},
  year = {2025},
  volume = {2025},
  pages = {23569--23592},
  url = {https://proceedings.iclr.cc/paper_files/paper/2025/hash/3b057de5a2e38bd8fa10201866c20dbf-Abstract-Conference.html}
}

@inproceedings{guo2022generating,
  title = {{{Generating Diverse and Natural 3D Human Motions from Text}}},
  author = {Guo, Chuan and Zou, Shihao and Zuo, Xinxin and Wang, Sen and Ji, Wei and Li, Xingyu and Cheng, Li},
  booktitle = {Proceedings of the IEEE/CVF Conference on Computer Vision and Pattern Recognition},
  year = {2022},
  pages = {5142--5151},
  doi = {10.1109/CVPR52688.2022.00509},
}

@inproceedings{zhang2023remodiffuse,
  title = {{{ReMoDiffuse: Retrieval-Augmented Motion Diffusion Model}}},
  author = {Zhang, Mingyuan and Guo, Xinying and Pan, Liang and Cai, Zhongang and Hong, Fangzhou and Li, Huirong and Yang, Lei and Liu, Ziwei},
  booktitle = {Proceedings of the IEEE/CVF International Conference on Computer Vision},
  year = {2023},
  pages = {364--373},
  doi = {10.1109/ICCV51070.2023.00040},
}

@inproceedings{guo2024momask,
  title = {{{MoMask: Generative Masked Modeling of 3D Human Motions}}},
  author = {Guo, Chuan and Mu, Yuxuan and Javed, Muhammad Gohar and Wang, Sen and Cheng, Li},
  booktitle = {Proceedings of the IEEE/CVF Conference on Computer Vision and Pattern Recognition},
  year = {2024},
  pages = {1900--1910},
  doi = {10.1109/CVPR52733.2024.00186},
}

@inproceedings{xiao2025motionstreamer,
  title = {{{MotionStreamer: Streaming Motion Generation via Diffusion-Based Autoregressive Model in Causal Latent Space}}},
  author = {Xiao, Lixing and Lu, Shunlin and Pi, Huaijin and Fan, Ke and Pan, Liang and Zhou, Yueer and Feng, Ziyong and Zhou, Xiaowei and Peng, Sida and Wang, Jingbo},
  booktitle = {Proceedings of the IEEE/CVF International Conference on Computer Vision},
  year = {2025},
  pages = {10086--10096},
  doi = {10.1109/ICCV51701.2025.00940},
}

@inproceedings{zhang2025primal,
  title = {{{PRIMAL: Physically Reactive and Interactive Motor Model for Avatar Learning}}},
  author = {Zhang, Yan and Feng, Yao and Cseke, Alp{\'a}r and Saini, Nitin and Bajandas, Nathan and Heron, Nicolas and Black, Michael J.},
  booktitle = {Proceedings of the IEEE/CVF International Conference on Computer Vision},
  year = {2025},
  pages = {12725--12736},
  doi = {10.1109/ICCV51701.2025.01182},
}

@inproceedings{liu2025rolling,
  title = {{Rolling Forcing: Autoregressive Long Video Diffusion in Real Time}},
  author = {Liu, Kunhao and Hu, Wenbo and Xu, Jiale and Shan, Ying and Lu, Shijian},
  booktitle = {International Conference on Learning Representations},
  year = {2026},
  volume = {2026},
  pages = {91177--91196},
  url = {https://proceedings.iclr.cc/paper_files/paper/2026/hash/935151cc6cb5d8b6816133b75233775a-Abstract-Conference.html}
}

@inproceedings{huang2025self,
  title = {{Self Forcing: Bridging the Train-Test Gap in Autoregressive Video Diffusion}},
  author = {Huang, Xun and Li, Zhengqi and He, Guande and Zhou, Mingyuan and Shechtman, Eli},
  booktitle = {Advances in Neural Information Processing Systems},
  year = {2025},
  volume = {38},
  pages = {167283--167308},
  doi = {10.52202/085713-5576},
  url = {https://proceedings.neurips.cc/paper_files/paper/2025/hash/f4823f831af67a3ef15e41a85434422a-Abstract-Conference.html},
  publisher = {Curran Associates, Inc.},
  note = {Main Conference Track}
}

@inproceedings{chen2024diffusion,
  title = {{Diffusion Forcing: Next-token Prediction Meets Full-Sequence Diffusion}},
  author = {Chen, Boyuan and Mart{\'\i} Mons{\'o}, Diego and Du, Yilun and Simchowitz, Max and Tedrake, Russ and Sitzmann, Vincent},
  booktitle = {Advances in Neural Information Processing Systems},
  year = {2024},
  volume = {37},
  pages = {24081--24125},
  doi = {10.52202/079017-0759},
  url = {https://proceedings.neurips.cc/paper_files/paper/2024/hash/2aee1c4159e48407d68fe16ae8e6e49e-Abstract-Conference.html},
  publisher = {Curran Associates, Inc.}
}

@inproceedings{zhou2019continuity,
  title = {{On the Continuity of Rotation Representations in Neural Networks}},
  author = {Zhou, Yi and Barnes, Connelly and Lu, Jingwan and Yang, Jimei and Li, Hao},
  booktitle = {2019 IEEE/CVF Conference on Computer Vision and Pattern Recognition (CVPR)},
  year = {2019},
  pages = {5738--5746},
  doi = {10.1109/CVPR.2019.00589},
  publisher = {IEEE}
}

@inproceedings{chen2023executing,
  title = {{{Executing your Commands via Motion Diffusion in Latent Space}}},
  author = {Chen, Xin and Jiang, Biao and Liu, Wen and Huang, Zilong and Fu, Bin and Chen, Tao and Yu, Gang},
  booktitle = {2023 IEEE/CVF Conference on Computer Vision and Pattern Recognition (CVPR)},
  pages = {18000--18010},
  year = {2023},
  doi = {10.1109/CVPR52729.2023.01726},
  eprint = {2212.04048},
  archivePrefix = {arXiv}
}

@article{banerjee2005optimality,
  title = {{On the Optimality of Conditional Expectation as a {Bregman} Predictor}},
  author = {Banerjee, Arindam and Guo, Xin and Wang, Hui},
  journal = {IEEE Transactions on Information Theory},
  year = {2005},
  volume = {51},
  number = {7},
  pages = {2664--2669},
  doi = {10.1109/TIT.2005.850145},
}

@article{gneiting2011making,
  title = {{Making and Evaluating Point Forecasts}},
  author = {Gneiting, Tilmann},
  journal = {Journal of the American Statistical Association},
  year = {2011},
  volume = {106},
  number = {494},
  pages = {746--762},
  doi = {10.1198/jasa.2011.r10138},
}
}

\clearpage
\appendix

\section{Partial Attention}
\label{sec:app_theory}

This section supports the Partial Attention construction in
\Cref{sec:partial}. We first recall FD1's streaming-locality result and
examine the information available to active predictions. We then establish
conditions for reusing finalized-history keys and values, followed by the
equivalence between shared-history packing and isolated training windows.

\subsection{Complete Active Information and Probability Flow}
\label{sec:proof_partial_diffusion}

Notation follows FloodDiffusion~\citep{cai2025flooddiffusion}. Let
$\bc^{0:T}$ collect the frame-aligned clean conditions of a sampling
trajectory, and let $\bz\in\mathbb R^{T\times d}$ be the clean diffused
sequence, $\bz\sim p_{\mathrm{data}}(\cdot\mid\bc^{0:T})$, with $d$ diffused
channels per frame. With $\bepsilon\sim\mathcal N(\mathbf 0,\bI)$
independent of $(\bz,\bc^{0:T})$, the streaming interpolation for frame
$k$ is
\begin{equation}
    \bx_t^k=\alpha_t^k\bz^k+\beta_t^k\bepsilon^k ,
    \qquad
    \dot{\bx}_t^k=\dot{\alpha}_t^k\bz^k+
                  \dot{\beta}_t^k\bepsilon^k .
    \label{eq:app_path}
\end{equation}
The frame-dependent coefficients follow the lower-triangular schedule
\begin{equation}
    \alpha_t^k=\mathrm{clamp}\bigl(t-\tfrac{k}{n_s},\,0,\,1\bigr),
    \qquad
    \beta_t^k=1-\alpha_t^k .
    \label{eq:app_tri_schedule}
\end{equation}
They induce marginals $p_t=\operatorname{Law}(\bx_t\mid\bc^{0:T})$.
Away from the schedule knots, finalized history, the active window, and
future are $H=[0,m(t))$, $A=[m(t),n(t))$, and $F=[n(t),T)$, where
\begin{equation}
\begin{aligned}
    m(t)&=\min\!\left(T,\max\!\left(0,\lfloor(t-1)n_s\rfloor+1\right)\right),\\
    n(t)&=\min\!\left(T,\lfloor tn_s\rfloor+1\right).
\end{aligned}
\label{eq:app_region_indices}
\end{equation}

Let $u_t(\bx,\bc^{0:T})$ be the marginal velocity field of this path.

\begin{assumption}[Causal Dependency; Assumption~3.7 of
\citealp{cai2025flooddiffusion}]
\label[assumption]{ass:app_causal_controls}
As in the streaming formulation of FloodDiffusion, a state prefix depends
only on the corresponding condition prefix:
\begin{equation}
    p_{\mathrm{data}}(\bz^{0:l}\mid\bc^{0:T})
    =p_{\mathrm{data}}(\bz^{0:l}\mid\bc^{0:l}),
    \qquad l<T.
    \label{eq:app_causal_controls}
\end{equation}
\end{assumption}

\begin{theorem}[Streaming Locality; Theorem~3.8 of
\citealp{cai2025flooddiffusion}]
\label{thm:app_streaming_locality}
Under \Cref{ass:app_causal_controls} and the schedule of
\Cref{eq:app_tri_schedule}, away from the finitely many knots,
\begin{equation}
    u_t(\bx_t,\bc^{0:T})
    =\bigl[\mathbf 0_H,\
      u_t^{m(t):n(t)}\bigl(\bx_t^{0:n(t)},\bc^{0:n(t)}\bigr),\
      \mathbf 0_F\bigr].
    \label{eq:app_exact_velocity}
\end{equation}
The field vanishes outside the active window, and its active block
depends only on the processed prefix.
\end{theorem}

We abbreviate the function $u_t^{m(t):n(t)}$ as $u_t^{A}$.

The following propositions justify causal history encoding while retaining
the complete active window as input.

\begin{proposition}[Causal history encoding can preserve complete prefix information]
\label[proposition]{prop:app_causal_encoding}
Fix $t$. There exist a causal history encoder
$\phi=(\phi_0,\ldots,\phi_{m(t)-1})$, whose component $\phi_k$ reads only
the rows $0{:}k{+}1$, and a measurable readout $g$ with
\begin{equation}
    u_t^{A}\bigl(\bx_t^{0:n(t)},\bc^{0:n(t)}\bigr)
    =g\bigl(\phi\bigl(\bx_t^{0:m(t)},\bc^{0:m(t)}\bigr),\
      \bx_t^{A},\bc^{A}\bigr).
    \label{eq:app_structured_form}
\end{equation}
Since history rows are saturated, $\bx_t^{0:m(t)}=\bz^{0:m(t)}$: the
encoder runs on the finalized clean prefix.
\end{proposition}

\begin{proof}
Take the causal projections
$\phi_k(\bx_t^{0:k+1},\bc^{0:k+1})=(\bx_t^k,\bc^k)$.
Their outputs and the active rows retain the entire processed prefix.
If $\rho$ restores its original arrangement, $g=u_t^{A}\circ\rho$.
\end{proof}

\begin{proposition}[Fully causal active predictions can be suboptimal]
\label[proposition]{prop:app_mask_necessity}
Fix a time $t$ away from the schedule knots with two active rows $j<k$.
There exists a data law satisfying \Cref{ass:app_causal_controls} under which
every predictor whose output at row $j$ reads only the rows $0{:}j{+}1$
has expected squared error exceeding that of the full-information
conditional mean $u_t^j$ by a positive constant independent of the predictor.
\end{proposition}

\begin{proof}
Take all conditions to be deterministic. Let $(z^j,z^k)$ be jointly
Gaussian with zero means, positive variances $\sigma_j^2,\sigma_k^2$, and
correlation $0<|\rho|<1$. All other clean coordinates are independent of
this pair and of each other; all Gaussian noise coordinates are independent
of each other and of the clean data. This law satisfies
\Cref{ass:app_causal_controls}. Both selected rows lie on the open ramp,
and the other processed rows provide no additional information about their
velocities. For this fixed $t$, write $\alpha_i=\alpha_t^i$ and
$\beta_i=\beta_t^i$, with the same convention for their derivatives. By Gaussian regression,
$u_t^j=\mathbb E[\dot x_t^j\mid x_t^j,x_t^k]=a\,x_t^j+b\,x_t^k$, with
\begin{equation}
    b=\frac{\operatorname{Cov}(\dot x_t^j,x_t^k\mid x_t^j)}
           {\operatorname{Var}(x_t^k\mid x_t^j)},
    \qquad
    \operatorname{Cov}(\dot x_t^j,x_t^k\mid x_t^j)
    =\frac{\alpha_k\rho\sigma_j\sigma_k\,\beta_j
      (\dot\alpha_j\beta_j-\dot\beta_j\alpha_j)}
      {\alpha_j^2\sigma_j^2+\beta_j^2}.
    \label{eq:app_gaussian_dependence}
\end{equation}
On the open ramp of \Cref{eq:app_tri_schedule},
$\dot\alpha_j\beta_j-\dot\beta_j\alpha_j=1$, so $b\neq0$.
For a predictor $f$ causal at row $j$, the additional rows in its prefix
are independent of $(x_t^j,x_t^k,\dot x_t^j)$. Projection onto functions
of that prefix therefore gives
\begin{equation}
    \mathbb E\bigl[(f(\bx_t^{0:j+1})-\dot x_t^j)^2\bigr]
    -\mathbb E\bigl[(u_t^j-\dot x_t^j)^2\bigr]
    \geq b^2\operatorname{Var}(x_t^k\mid x_t^j)>0 .
    \label{eq:app_masked_excess}
\end{equation}
For vector-valued frames, embed this construction in one coordinate and
take all remaining coordinates to be independent.
\end{proof}

\subsection{History-Prefix Invariance}
\label{sec:proof_partial}

Consider a layered computation with Partial Attention and row-wise
normalization, residual, and feed-forward operations.

\begin{lemma}[Layer-wise suffix independence]
\label[lemma]{lem:history_suffix}
Assume that the input embedding of each history row $i$ depends only on
rows $0{:}i{+}1$, and that conditioning operations, including cross-attention,
read only the conditions attached to that row.
Fix a history position $j$ and two inputs that agree on every row $i\le j$:
state tokens, noise metadata, position IDs, and per-row conditions. Under
Partial Attention, the hidden state and the layer-wise keys and values at
$j$ are identical at every layer in deterministic evaluation, whatever the
later suffix. They also equal those of an inclusive causal execution with
the same prefix inputs.
\end{lemma}

\begin{proof}
Induction over layers. A history query $i\le j$ reads only keys
$0{:}i{+}1\subseteq0{:}j{+}1$, its cross-attention reads only the condition
attached to row $i$, and normalization, residual, and feed-forward act
row-wise; so equality of rows $0{:}j{+}1$ before a layer implies equality
after it, and the input layer is equal by hypothesis. Inclusive causal
attention has the same visible sets on these history rows.
\end{proof}

Thus finalized-history K/V can be reused while new rows are encoded; a
prompt switch affecting only later rows preserves the cache.

\subsection{Multi-Window Packing}
\label{sec:proof_multiwindow}

For $K$ copied windows $A_1,\ldots,A_K$ with start rows
$s_1,\ldots,s_K$ over one reference sequence $C$, the packed mask is
\begin{equation}
    \mathcal V(q)=
    \begin{cases}
        C_{\le q}, & q\in C,\\
        C_{<s_w}\cup A_w, & q\in A_w ,
    \end{cases}
    \label{eq:app_packed_visible}
\end{equation}
so no window reads another, even when their source ranges overlap.

\begin{proposition}[Packed/isolated equivalence]
\label[proposition]{prop:app_packed}
With untruncated history, deterministic evaluation (or coupled stochastic
masks), and identical reference content, target content, metadata, and
noise draw between each packed window and its isolated execution, every
prediction on $A_w$ equals its isolated prediction in exact arithmetic.
\end{proposition}

\begin{proof}
Reference rows compute their states causally, unaffected by
every packed target (\Cref{lem:history_suffix}). A query in $A_w$
therefore sees exactly the computation graph of its isolated call
(\Cref{eq:app_packed_visible}) and has no path to another window; a
layer-wise induction equates the hidden states on $A_w$.
\end{proof}

Let $\mathcal I_w\subseteq A_w$ contain the valid prediction rows of copy
$w$, and let $N_w=|\mathcal I_w|$. For $N_w>0$, define its mean per-row
quadratic loss by
\[
    L_w=\frac{1}{N_w}\sum_{i\in\mathcal I_w}
        \frac{1}{d}\mathbf e_{w,i}^{\top}W_\gamma\mathbf e_{w,i},
\]
where $\mathbf e_{w,i}$ is the prediction residual at row $i$ of copy $w$,
with $W_\gamma$ from \Cref{eq:mesh_quadratic}.
Windows with no valid prediction rows are omitted from the loss.
Provided $\sum_{w=1}^{K}N_w>0$, the packed objective is
\begin{equation}
    L_{\mathrm{packed}}
    =\frac{\sum_{w:N_w>0}N_wL_w}{\sum_{w=1}^{K}N_w}.
    \label{eq:app_token_weight}
\end{equation}
Thus packing preserves the isolated predictions and their
token-count-weighted objective.
The equivalence is conditional on the realized reference inputs. If
reference perturbations are used during training, their states and noise
metadata must also be shared by the packed and isolated executions.

\section{Diffusion-Compatible Regression Losses}
\label{sec:app_loss_characterization}

This section expands Section~\ref{sec:mesh_objective}. We first state the
conditional-mean target of velocity regression and the necessary-and-sufficient
Bregman criterion for preserving it. We then show why direct nonlinear
geometric losses need not preserve this target, construct the fixed FK-induced
quadratic, and give its two motion-state implementations. Matrix diagnostics
and a controlled rotation experiment illustrate the construction and its
effect on sampling.

\subsection{Conditional Targets Train the Posterior Average}
\label{sec:app_conditional_training}

We write $\dot{\bx}_t$ for the sampled velocity denoted by $\bv_t$ in the main text. The active velocity is the conditional mean of the sampled interpolation
target (Definition~3.3 of \citealp{cai2025flooddiffusion}):
\begin{equation}
    u_t^{A}\bigl(\bx_t^{0:n(t)},\bc^{0:n(t)}\bigr)
    =\mathbb E\bigl[\dot{\bx}_t^{A}\,\big|\,
      \bx_t^{0:n(t)},\bc^{0:n(t)}\bigr]
    \label{eq:app_regression_variables}
\end{equation}
Fix $t$ away from the schedule knots with $A\ne\varnothing$, and write
$X_t=(\bx_t^{0:n(t)},\bc^{0:n(t)})$. We apply the loss to each active row:
$Y_k=\dot{\bx}_t^k$ and its prediction $f_t^k(X_t)$ take values in an open
convex set $\Omega\subseteq\mathbb R^d$, with $d$ predicted channels per
frame. The per-time risk is
$\mathcal R_t(f)=\mathbb E[|A|^{-1}\sum_{k\in A}
\ell(Y_k,f_t^k(X_t))]$.

\begin{definition}[Mean consistency]
\label[definition]{def:app_mean_consistency}
A loss $\ell(y,a)$ is \emph{mean-consistent} on $\Omega$ if, for every
finitely supported distribution $P$ on $\Omega$, its mean
$\mu_P=\mathbb E_{y\sim P}[y]\in\Omega$ minimizes
$R_P(a)=\mathbb E_{y\sim P}[\ell(y,a)]$. It is \emph{strictly}
mean-consistent if this minimizer is unique.
\end{definition}

\begin{proposition}[Strictly mean-consistent training recovers the posterior
average]
\label[proposition]{prop:app_conditional_training}
Let the risks be finite, and suppose that for each active row $k$, the
conditional mean $\mathbb E[Y_k\mid X_t]$ uniquely minimizes its conditional
expected loss under $\ell$. Then $u_t^{A}$ uniquely minimizes
$\mathcal R_t$ up to null sets. Applying this result at almost every
diffusion time recovers the streaming velocity field, whose zero extension
reproduces the conditional data law under the sampling formulation of
\citet{cai2025flooddiffusion} (Theorems~3.4 and~3.8).
\end{proposition}

\begin{proof}
The tower property expresses the excess risk as the expected average of
the per-row conditional excess risks. Each is nonnegative and vanishes
only at $f_t^k(X_t)=u_t^k(X_t)$ by hypothesis and
\Cref{eq:app_regression_variables}.
The sampling claim follows from \Cref{thm:app_streaming_locality} and the
conditional-generation theorem of \citet{cai2025flooddiffusion}.
\end{proof}

\Cref{def:app_mean_consistency} quantifies over finitely supported laws
because they suffice for the necessity direction below. For a Bregman loss
$\ell(Y,a)=D_\varphi(Y,a)+c(Y)$, the identity in
\Cref{eq:app_bregman_risk_decomposition} extends the conclusion to laws
whose mean lies in $\Omega$ and for which
$\mathbb E\|Y\|$, $\mathbb E|\varphi(Y)|$, and $\mathbb E|c(Y)|$ are finite.
For the quadratic losses below, a finite second moment suffices.

\subsection{A Necessary and Sufficient Characterization}

\begin{lemma}[Identification of a mean-consistent gradient]
\label[lemma]{lem:app_mean_identification}
Let $g:\Omega\times\Omega\to\mathbb R^d$ be continuous. Suppose that, for
every $a\in\Omega$ and every finite collection
$\{(\lambda_i,y_i)\}_{i=1}^n$ with $\lambda_i>0$,
$\sum_i\lambda_i=1$, and $\sum_i\lambda_i y_i=a$, one has
$\sum_i\lambda_i g(y_i,a)=0$. Then there is a matrix-valued function $H(a)$
such that
\begin{equation}
    g(y,a)=H(a)(a-y),
    \qquad y,a\in\Omega.
    \label{eq:app_mean_identification}
\end{equation}
\end{lemma}

\begin{proof}
Fix $a$ and write $h(v)=g(a+v,a)$ on the open convex set
$V=\Omega-a$, which contains the origin. The one-point distribution gives
$h(0)=0$. On a sufficiently small ball $B(0,r)\subset V$, a two-point
zero-mean distribution supported on $v$ and $-v$ gives
$h(-v)=-h(v)$. Whenever $v,w,v+w$ and their negatives lie in that ball,
the equal-weight distribution on $v,w,-(v+w)$ gives
$h(v)+h(w)+h(-(v+w))=0$, hence $h(v+w)=h(v)+h(w)$.
Repeated addition gives rational homogeneity;
continuity gives real homogeneity. Therefore $h(v)=A(a)v$ near zero for a
linear map $A(a)$.

For arbitrary $v\in V$, choose $\rho>0$ small enough that
$-\rho v\in B(0,r)$. The distribution assigning masses
$\rho/(1+\rho)$ and $1/(1+\rho)$ to $v$ and $-\rho v$ has mean zero, so
$\rho h(v)+h(-\rho v)=0$. Local linearity makes the second term
$-\rho A(a)v$, hence
$h(v)=A(a)v$ throughout $V$. Setting $H(a)=-A(a)$ proves
\Cref{eq:app_mean_identification}.
\end{proof}

The theorem below states the smooth multivariate form of the classical
mean-consistency criterion
\citep{banerjee2005optimality,gneiting2011making}.

\begin{theorem}[Necessary and sufficient loss characterization]
\label{thm:app_mean_bregman}
Let $\ell:\Omega\times\Omega\to\mathbb R$ be three times continuously
differentiable. It is mean-consistent for every finitely supported law whose
support and mean lie in $\Omega$ if and only if
\begin{equation}
    \ell(y,a)=D_\varphi(y,a)+c(y),
    \qquad
    D_\varphi(y,a)
    =\varphi(y)-\varphi(a)-\nabla\varphi(a)^\top(y-a),
    \label{eq:app_target_first_bregman}
\end{equation}
for a convex potential $\varphi\in C^4(\Omega)$ and a target-only function
$c\in C^3(\Omega)$. The loss is strictly mean-consistent if and only if
$\varphi$ is strictly convex.
\end{theorem}

\begin{proof}
\emph{Necessity.}
Let $g(y,a)=\nabla_a\ell(y,a)$. Because $\Omega$ is open, mean
consistency for every finitely supported law $P$ with mean $a$ gives
$\mathbb E_{y\sim P}[g(y,a)]=0$. \Cref{lem:app_mean_identification}
therefore yields
$g(y,a)=H(a)(a-y)$. For fixed $y_0\in\Omega$ and sufficiently small $h>0$,
$H(a)e_j=[g(y_0,a)-g(y_0+h e_j,a)]/h$ for each coordinate basis vector
$e_j$; hence $H\in C^2$ since $g\in C^2$.
For the law degenerate at $a$, the action $a$ minimizes
$\ell(a,\cdot)$, so
\[
    H(a)=\left.\nabla_a^2\ell(y,a)\right|_{y=a}\succeq0;
\]
Thus $H$ is symmetric. Differentiating $g$ and using symmetry of the
action Hessian gives
\[
    \sum_k\bigl(\partial_jH_{ik}(a)-\partial_iH_{jk}(a)\bigr)(a_k-y_k)=0.
\]
This holds for every $y$ in the open set $\Omega$, so each coefficient
vanishes: for every $i,j,k$,
\begin{equation}
    \partial_j H_{ik}(a)=\partial_i H_{jk}(a).
    \label{eq:app_hessian_integrability}
\end{equation}
Together with $H_{jk}=H_{kj}$, this implies
$\partial_kH_{ij}=\partial_jH_{ik}$, the row-wise integrability condition.

Fix $a_0\in\Omega$, put $r=a-a_0$, and define
\begin{align}
    p_i(a)
    &=\int_0^1\sum_j H_{ij}(a_0+t r)r_j\,dt, \\
    \varphi(a)
    &=\int_0^1 p(a_0+t r)^\top r\,dt .
    \label{eq:app_potential_integrals}
\end{align}
Differentiating the first integral and using
$\partial_kH_{ij}=\partial_jH_{ik}$ gives
\[
    \partial_k p_i(a)
    =\int_0^1\!\left[t\frac{d}{dt}H_{ik}(a_0+tr)
                     +H_{ik}(a_0+tr)\right]dt
    =H_{ik}(a).
\]
Since $H$ is symmetric, the same calculation for the second integral gives
$\partial_k\varphi(a)=p_k(a)$ and hence $\nabla^2\varphi(a)=H(a)\succeq0$.
Thus $\varphi$ is convex and lies in $C^4$ because $H\in C^2$.
Direct differentiation gives
\[
    \nabla_aD_\varphi(y,a)=\nabla^2\varphi(a)(a-y)=g(y,a).
\]
Consequently $\nabla_a[\ell(y,a)-D_\varphi(y,a)]=0$. Since $\Omega$ is
connected, the bracket equals $c(y)=\ell(y,y)\in C^3$, proving
\Cref{eq:app_target_first_bregman}.

\emph{Sufficiency.}
For any law $P$ with mean $\mu_P$ for which the expectations exist,
\begin{align}
    \mathbb E_{y\sim P}[D_\varphi(y,a)]
    &=\mathbb E_{y\sim P}[\varphi(y)]-\varphi(a)
      -\nabla\varphi(a)^\top(\mu_P-a),\\
    \mathbb E_{y\sim P}[D_\varphi(y,a)]
    -\mathbb E_{y\sim P}[D_\varphi(y,\mu_P)]
    &=D_\varphi(\mu_P,a)\geq0.
    \label{eq:app_bregman_risk_decomposition}
\end{align}
The target-only term $c(y)$ cancels. Thus the mean is a minimizer, and strict
convexity of $\varphi$ makes it unique. Conversely, strict mean consistency for
every degenerate law requires $D_\varphi(y,a)>0$ whenever $y\ne a$, which is the
strict supporting-plane characterization of differentiable strict convexity.
\end{proof}

The argument order in \Cref{eq:app_target_first_bregman} is essential: the
random diffusion target is the first argument and the network prediction is
the second.

\subsection{Direct Nonlinear Losses Need Not Be Mean-Consistent}

The model predicts velocity rather than a clean state, so a loss through a
nonlinear map of the clean state must first reconstruct that state. For one
active token, write
\begin{equation}
    \bx_t=\alpha_t\bz+\beta_t\bepsilon,\qquad
    \dot{\bx}_t=\dot\alpha_t\bz+\dot\beta_t\bepsilon,\qquad
    \Delta_t=\alpha_t\dot\beta_t-
                  \dot\alpha_t\beta_t.
    \label{eq:app_velocity_pair}
\end{equation}
At an active interior diffusion time with $\Delta_t\ne0$ and
$\beta_t\ne0$, a proposed velocity $a$ implies the clean-state estimate
\begin{equation}
    R_t(a)=\frac{\dot\beta_t\bx_t-\beta_t a}{\Delta_t},
    \qquad R_t(\dot{\bx}_t)=\bz.
    \label{eq:app_clean_from_velocity}
\end{equation}
For our linear active ramp, $\dot\alpha_t=1$ and $\dot\beta_t=-1$, so
$\Delta_t=-1$ and $R_t(a)=\bx_t+\beta_ta$, as in the main text.
Let $\boldsymbol\eta$ be a clean decoder context from the processed
prefix, and let
$\mathcal S_{\boldsymbol\eta}:R_t(\Omega)\to\mathbb R^{m}$ be any differentiable
nonlinear map from the clean state to a geometric output. Define the
nonlinear composite $\mathcal K_t=\mathcal S_{\boldsymbol\eta}\circ R_t$.
The geometric decoder is denoted by $F_{\boldsymbol\eta}$ in the main text.

\begin{proposition}[A direct nonlinear loss is not generally mean-consistent]
\label[proposition]{prop:app_direct_mesh_inconsistent}
Let $\mathcal K_t:\Omega\to\mathbb R^{m}$ be the differentiable nonlinear
map above and let $M\succeq0$ be a fixed output metric. Define
$\ell_{\mathrm{dir}}(y,a)
=\|\mathcal K_t(y)-\mathcal K_t(a)\|_{M}^2$, where
$\|v\|_M^2=v^\top Mv$. For a target law $P$ on $\Omega$ with finite mean
$\mu_P\in\Omega$ and
$\mathbb E_{y\sim P}\|\mathcal K_t(y)\|_2^2<\infty$,
\begin{equation}
    \nabla_a\,\mathbb E_{y\sim P}[\ell_{\mathrm{dir}}(y,a)]
    =2J_{\mathcal K_t}(a)^\top M
      \left(\mathcal K_t(a)-\mathbb E_{y\sim P}[\mathcal K_t(y)]\right),
    \label{eq:app_direct_mesh_gradient}
\end{equation}
which need not vanish at $a=\mu_P$.
\end{proposition}

\begin{proof}
The moment assumption makes the risk finite for every $a\in\Omega$.
Expanding the square gives
\[
\begin{aligned}
R_P(a)
&=\mathbb E_{y\sim P}
  [\mathcal K_t(y)^\top M\mathcal K_t(y)]\\
&\quad-2\mathcal K_t(a)^\top M
  \mathbb E_{y\sim P}[\mathcal K_t(y)]
  +\mathcal K_t(a)^\top M\mathcal K_t(a).
\end{aligned}
\]
Differentiating this finite expression with respect to $a$ gives
\Cref{eq:app_direct_mesh_gradient}.
A single point rigidly rotated in a plane already makes the failure
explicit. The
nonzero affine scale in \Cref{eq:app_clean_from_velocity} can be absorbed into
the scalar angle coordinate. Let
\[
    K(\theta)=(\cos\theta,\sin\theta),\qquad M=I_2,
\]
and let $P$ put mass $\tfrac34$ at $0$ and $\tfrac14$ at $\tfrac\pi2$, so
$\mu_P=\pi/8$ and $\mathbb E_{y\sim P}[K(y)]=(3/4,1/4)$. Since
$J_K(\mu_P)=(-\sin\mu_P,\cos\mu_P)^\top$, the derivative of the direct
expected loss at the mean equals
\[
    2J_K(\mu_P)^\top(K(\mu_P)-\mathbb E_{y\sim P}[K(y)])
    =\tfrac12\left(3\sin\tfrac\pi8-\cos\tfrac\pi8\right)\ne0.
\]
Adding this loss with any positive coefficient to a mean-consistent quadratic
leaves a nonzero gradient at $\mu_P$, so the sum is not generally
mean-consistent.
\end{proof}

\subsection{FK-Induced Quadratic Objective}
\label{sec:app_mesh_quadratic}

\paragraph{\normalfont The dataset-averaged response Gram matrix.}
Here $\boldsymbol\xi$ denotes the standardized clean state written as
$\bz$ in the main text.
Let
$J_{\boldsymbol\eta,\boldsymbol\xi}\in\mathbb R^{m\times d}$ assign a
linear response into an output space with fixed metric $M\succeq0$ to each
clean context $\boldsymbol\eta$ and state $\boldsymbol\xi$, expressed in
the standardized training coordinate. The 263-D features use a designed
kinematic joint response. For the 138-D state, $J$ is the Jacobian of the
surface decoder $F_{\boldsymbol\eta}$,
$\partial F_{\boldsymbol\eta}(\boldsymbol\xi)/\partial\boldsymbol\xi$.
Assume
$0<\mathbb E\|M^{1/2}J_{\boldsymbol\eta,\boldsymbol\xi}\|_F^2<\infty$.
Define the dataset-averaged Gram matrix and its normalization
\begin{equation}
    \overline G
    =\mathbb E_{(\boldsymbol\eta,\boldsymbol\xi)}
      \bigl[J_{\boldsymbol\eta,\boldsymbol\xi}^\top
        MJ_{\boldsymbol\eta,\boldsymbol\xi}\bigr],
    \qquad
    W_{\mathrm{FK}}
    =\frac{d\,\overline G}
       {\operatorname{tr}(\overline G)},
    \qquad
    W_\gamma=\frac{I_{d}+\gamma W_{\mathrm{FK}}}{1+\gamma},
    \qquad \gamma\geq0,
    \label{eq:app_fixed_mesh_weight}
\end{equation}
where the expectation runs over the training entries used for matrix
estimation. Every summand is PSD, and
$\operatorname{tr}(\overline G)
=\mathbb E\|M^{1/2}J_{\boldsymbol\eta,\boldsymbol\xi}\|_F^2\in(0,\infty)$,
so
$W_\gamma\succeq(1+\gamma)^{-1}I_{d}\succ0$ and
$\operatorname{tr}(W_{\mathrm{FK}})=\operatorname{tr}(W_\gamma)=d$. Because
$\boldsymbol\xi$ is the standardized training coordinate, the response
composes with
$\mathbf q=\boldsymbol\mu_{\mathrm{ch}}
+\boldsymbol\sigma_{\mathrm{ch}}\odot\boldsymbol\xi$, so every column
already contains the corresponding channel standard deviation.

\begin{theorem}[The FK-induced quadratic preserves the diffusion field]
\label{thm:app_fixed_mesh_consistency}
Let $W_\gamma$ be the deterministic matrix in
\Cref{eq:app_fixed_mesh_weight}. For every $\gamma\geq0$, the loss
\begin{equation}
    \ell_\gamma(y,a)=\frac{1}{d}(y-a)^\top W_\gamma(y-a)
    \label{eq:app_fixed_mesh_loss}
\end{equation}
is a strictly mean-consistent target-first Bregman divergence: under every
target law with finite second moment its unique optimum is that law's
mean. Together with
\Cref{prop:app_conditional_training}, its population solution yields the
original probability path and terminal distribution.
\end{theorem}

\begin{proof}
The positive-definite matrix $W_\gamma$ defines the strictly convex potential
\begin{equation}
    \varphi_\gamma(u)=\frac{1}{d}u^\top W_\gamma u,\qquad
    \nabla\varphi_\gamma(u)=\frac{2}{d}W_\gamma u,\qquad
    \nabla^2\varphi_\gamma(u)=\frac{2}{d}W_\gamma\succ0.
    \label{eq:app_mesh_potential}
\end{equation}
Substitution into \Cref{eq:app_target_first_bregman} gives
$D_{\varphi_\gamma}(y,a)=\ell_\gamma(y,a)$, and
\Cref{eq:app_bregman_risk_decomposition} yields
\begin{equation}
    \mathbb E_{y\sim P}[\ell_\gamma(y,a)]
    =\mathbb E_{y\sim P}[\ell_\gamma(y,\mu_P)]
      +\frac{1}{d}(a-\mu_P)^\top W_\gamma(a-\mu_P).
    \label{eq:app_fixed_mesh_risk}
\end{equation}
The final term vanishes only at $a=\mu_P$.
\Cref{prop:app_conditional_training} then gives the sampling statement.
\end{proof}

When $J$ is the Jacobian of a decoder $F_{\boldsymbol\eta}$, the local
Taylor expansion
\begin{equation}
  \|F_{\boldsymbol\eta}(\boldsymbol\xi+\delta)
      -F_{\boldsymbol\eta}(\boldsymbol\xi)
    \|_{M}^2
  =\delta^\top J_{\boldsymbol\eta,\boldsymbol\xi}^\top
   MJ_{\boldsymbol\eta,\boldsymbol\xi}\delta
   +o(\|\delta\|_2^2)
  \label{eq:app_mesh_taylor}
\end{equation}
shows the local output geometry retained by $W_{\mathrm{FK}}$. Since
$D_aR_t=-\beta_t I_d/\Delta_t$, mapping a velocity residual back to
clean coordinates scales this metric by $\beta_t^2/\Delta_t^2$.
We omit this positive, target-independent factor, reweighting time/frame
contributions while preserving their conditional-mean optima.

A target-dependent per-sample matrix is not interchangeable with
\Cref{eq:app_fixed_mesh_weight}. For
$\ell(y,a)=(y-a)^\top W(y)(y-a)$ and $Y\sim P$ with finite mean,
let $W$ be measurable with
$W(y)\succeq0$, $\mathbb E[\|W(Y)\|_F(1+\|Y\|_2^2)]<\infty$, and
$\mathbb E[W(Y)]\succ0$. Minimizing over $a\in\mathbb R^d$ gives
\begin{equation}
    \mathbb E_{y\sim P}[W(y)]\,a
    =\mathbb E_{y\sim P}[W(y)\,y],
    \label{eq:app_target_dependent_weight}
\end{equation}
Thus the unique minimizer is
$a^*=(\mathbb E[W(Y)])^{-1}\mathbb E[W(Y)Y]$, which generally differs from
$\mu_P$. Matrices $W(X)\succ0$ depending only on predictor inputs preserve
$\mathbb E[Y\mid X]$. The offline average in
\Cref{eq:app_fixed_mesh_weight} avoids FK and Jacobian evaluations during training.

\subsection{Response Operators for the Two Motion States}
\label{sec:app_fk_implementations}

The common construction uses different geometric responses for the two
motion states: a surface-decoder Jacobian for 138-D SEED motion and a
designed joint-response operator for native 263-D HumanML3D features.
Both are expressed in standardized predicted coordinates and averaged
offline before training. The path-conditioned branches restrict the raw
Gram matrix to predicted channels before trace normalization.

\paragraph{\normalfont Surface instantiation: the 138-D state.}
Here $F_{\boldsymbol\eta}$ is neutral-SMPL-H mesh FK with $N_v$ template
vertices. Let
$\mathcal T$ be the triangle set of the neutral-shape template,
and let $A_f^{(0)}$ be the rest-template area of
triangle $f$. The lumped and normalized area assigned to vertex $v$ is
\begin{equation}
    \widetilde a_v=\frac13\sum_{f\in\mathcal T:\,v\in f}A_f^{(0)},
    \qquad
    a_v=\frac{\widetilde a_v}{\sum_{u=1}^{N_v}\widetilde a_u},
    \qquad \sum_{v=1}^{N_v}a_v=1,
    \label{eq:app_mesh_vertex_areas}
\end{equation}
and the area-weighted squared vertex-displacement metric is
\begin{equation}
    M
    =\operatorname{diag}(a_1,\ldots,a_{N_v})\otimes I_3,
    \qquad
    \|\delta\mathbf v\|_{M}^2
    =\sum_{v=1}^{N_v}a_v\|\delta\mathbf v_v\|_2^2.
    \label{eq:app_mesh_area_metric}
\end{equation}
The metric uses fixed rest-template areas and vertex correspondences.
Area weighting approximates mean-squared surface displacement, and unit-sum
normalization fixes the objective scale with $\operatorname{tr}(M)=3$.

Let $\mathcal N_m$ be the fixed channel-wise standardization of the physical
state $\mathbf m=[\bc\Vert\bs]$. Its restrictions to the root and
articulated channels are $\mathcal N_c$ and $\mathcal N_s$, respectively:
$\mathcal N_m(\mathbf m)=[\mathcal N_c(\bc)\Vert\mathcal N_s(\bs)]$.
Let $b\in\{\mathrm{FD2},\mathrm{Path}\}$ index the text-only and
path-conditioned branches, with $d_{\mathrm{FD2}}=138$ and
$d_{\mathrm{Path}}=135$. For valid frame $i$ of training motion $c$, let
$h_{c,i}$ contain the teacher-forced predecessor root position and heading (or
the initial pose at the first frame), and define
\begin{equation}
  (\boldsymbol\xi_{c,i}^{(b)},\boldsymbol\eta_{c,i}^{(b)},d_b)=
  \begin{cases}
    (\mathcal N_m(\mathbf m_{c,i}),h_{c,i},138), & b=\mathrm{FD2},\\
    (\mathcal N_s(\bs_{c,i}),(h_{c,i},\bc_{c,i}),135), & b=\mathrm{Path}.
  \end{cases}
  \label{eq:app_mesh_branch_context}
\end{equation}
The context $\boldsymbol\eta_{c,i}^{(b)}$ remains in physical coordinates.
The FK function
$F_{\boldsymbol\eta_{c,i}^{(b)}}
(\boldsymbol\xi_{c,i}^{(b)})\in\mathbb R^{3N_v}$ includes inverse channel
normalization, root--pelvis reconstruction, the neutral-shape SMPL-H pose
correctives, and linear blend skinning (LBS).
Differentiating its final world-space vertices only with respect to the
predicted block gives
\begin{equation}
    J_{c,i}^{(b)}
    =\frac{\partial
       F_{\boldsymbol\eta_{c,i}^{(b)}}
       (\boldsymbol\xi_{c,i}^{(b)})}
      {\partial\boldsymbol\xi_{c,i}^{(b)}}
    \in\mathbb R^{3N_v\times d_b}.
    \label{eq:app_mesh_jacobian}
\end{equation}
Let $\mathcal C_{\mathrm{train}}$ be the calibration set and let
$\mathcal I_c$ contain its selected valid non-initial frames for entry $c$.
For the full-training-set HumanML3D calibration reported in
\Cref{tab:mesh_matrix_diagnostics}, these are all valid non-initial frames
from the 21,466 listed training entries. The average is
\begin{equation}
    \overline G_{\mathrm{FK}}^{(b)}
    =\frac{1}{\sum_{c\in\mathcal C_{\mathrm{train}}}|\mathcal I_c|}
       \sum_{c\in\mathcal C_{\mathrm{train}}}
       \sum_{i\in\mathcal I_c}
       (J_{c,i}^{(b)})^\top MJ_{c,i}^{(b)}
    \succeq0.
    \label{eq:app_mesh_pullback}
\end{equation}
Each branch is normalized independently through
\Cref{eq:app_fixed_mesh_weight} with $d=d_b$, yielding
$W_{\mathrm{FK}}^{(b)}$ and $W_\gamma^{(b)}$.

For text-only FD2, the Jacobian in \Cref{eq:app_mesh_jacobian} contains all
138 columns. In particular, the planar root displacement moves every surface
vertex, while the heading increment changes both the displacement direction
and the orientation of the complete surface. For FD2-Path, the predecessor
context and current root channels are fixed, and the Jacobian contains only
the predicted 135-D articulated block. With the same calibration poses, this
is equivalently the articulated principal block of the \emph{raw} text-only
pullback,
\begin{equation}
    \overline G_{\mathrm{FK}}^{(\mathrm{Path})}
    =\left(\overline G_{\mathrm{FK}}^{(\mathrm{FD2})}\right)_
      {\mathrm{art},\mathrm{art}}.
    \label{eq:app_path_mesh_principal_block}
\end{equation}
It is then trace-normalized with $d_{\mathrm{Path}}=135$, independently of
the 138-D branch.

\paragraph{\normalfont Joint instantiation: the 263-D features.}
Here the prediction is the native 263-D feature frame, and the response
is a designed per-frame operator
$J_{\mathrm{kin}}\in\mathbb R^{66\times263}$ into the 22 joint centers,
under the mean-squared joint-displacement metric
$M=\tfrac{1}{22}I_{66}$. It maps perturbations of the root channels and the
21 stored root-relative joint positions $\mathbf p_j$ to articulated
joint displacements. The
root yaw channel rotates the complete root-relative skeleton about the
vertical axis; the planar displacement and height channels translate every
joint rigidly. A perturbation $\delta\mathbf p_j$ of stored joint $j$
induces the minimum swing rotation of its incoming bone
$\mathbf b_j=\mathbf p_j-\mathbf p_{\mathrm{pa}(j)}$,
\begin{equation}
    \delta\boldsymbol\omega_j
    =\frac{\mathbf b_j\times\delta\mathbf p_j}{\|\mathbf b_j\|_2^2},
    \qquad
    \delta\mathbf p_k^{\mathrm{out}}
    =\delta\boldsymbol\omega_j\times
     \bigl(\mathbf p_k-\mathbf p_{\mathrm{pa}(j)}\bigr)
    \quad\text{for every $k$ in the subtree of $j$},
    \label{eq:app_joint_subtree_response}
\end{equation}
where $\mathrm{pa}(j)$ is the parent of $j$. The output response retains
the component of $\delta\mathbf p_j$ perpendicular to $\mathbf b_j$ and
propagates the induced swing through the subtree. A degenerate bone
contributes zero response for that joint. Velocity,
rotation, and contact channels receive no response, so their columns
vanish and the dataset-averaged Gram matrix has rank at
most $67$; the identity term in $W_\gamma$ restores positive definiteness. The
average assigns equal weight to every frame across 23,384 training entries,
including the additional HumanAct data, and
\Cref{eq:app_fixed_mesh_weight} normalizes with $d=263$. The
path-conditioned variant removes the three conditioned root
channels---yaw and planar displacement, while the root height stays
diffused---as the principal block $[3{:}263)$ of the raw Gram matrix,
renormalized with $d=260$.

\subsection{Spectral Structure of the FK Matrices}

The diagnostics in this subsection use the full-training-set HumanML3D
calibrations of the 138-D and 263-D representations. They describe the
geometric response matrices for these calibration data.

For a normalized FK matrix
$W_{\mathrm{FK}}$, define the effective rank
$r_{\mathrm{eff}}=\exp(-\sum_j p_j\log p_j)$, where
$p_j=\lambda_j(W_{\mathrm{FK}})/\operatorname{tr}(W_{\mathrm{FK}})$
over the nonzero PSD spectrum, and the off-diagonal ratio
$r_{\mathrm{off}}=\|W_{\mathrm{FK}}-
\operatorname{diag}(W_{\mathrm{FK}})\|_F/\|W_{\mathrm{FK}}\|_F$.

\begin{table}[!ht]
\centering
\scriptsize
\caption{Spectral properties of the HumanML3D-calibrated FK matrices.
Values describe $W_{\mathrm{FK}}$ before forming $W_\gamma$; by construction
$\operatorname{tr}(W_{\mathrm{FK}})=d$. Numerical rank counts eigenvalues
above $10^{-8}\lambda_{\max}$. Root share is the fraction of trace in the
three path-conditioned root channels.}
\label{tab:mesh_matrix_diagnostics}
\setlength{\tabcolsep}{2pt}
\fitwidth{\textwidth}{%
\begin{tabular}{l c c c c c c c c}
\toprule
Branch & $d$ & $\operatorname{tr}(\overline G_{\mathrm{FK}})$ & Numerical rank & $\lambda_{\min}(W_{\mathrm{FK}})$ & $\lambda_{\max}(W_{\mathrm{FK}})$ & $r_{\mathrm{eff}}$ & $r_{\mathrm{off}}$ & Root trace share \\
\midrule
Kinematic response, FD2 & 263 & 0.59209720 & 67/263 & $0$ & 103.9970 & 8.2594 & 0.6302 & 0.098\% \\
Kinematic response, FD2-Path & 260 & 0.59151672 & 64/260 & $0$ & 102.9012 & 8.2405 & 0.6300 & -- \\
Surface pullback, FD2 & 138 & 0.03065992 & 136/138 & $0$ & 22.8292 & 12.9788 & 0.5543 & 0.842\% \\
Surface pullback, FD2-Path & 135 & 0.03040161 & 135/135 & $2.280{\times}10^{-6}$ & 22.0458 & 13.0258 & 0.5405 & -- \\
\bottomrule
\end{tabular}
}
\end{table}

\begin{figure}[!ht]
    \centering
    \includegraphics[width=0.95\textwidth]{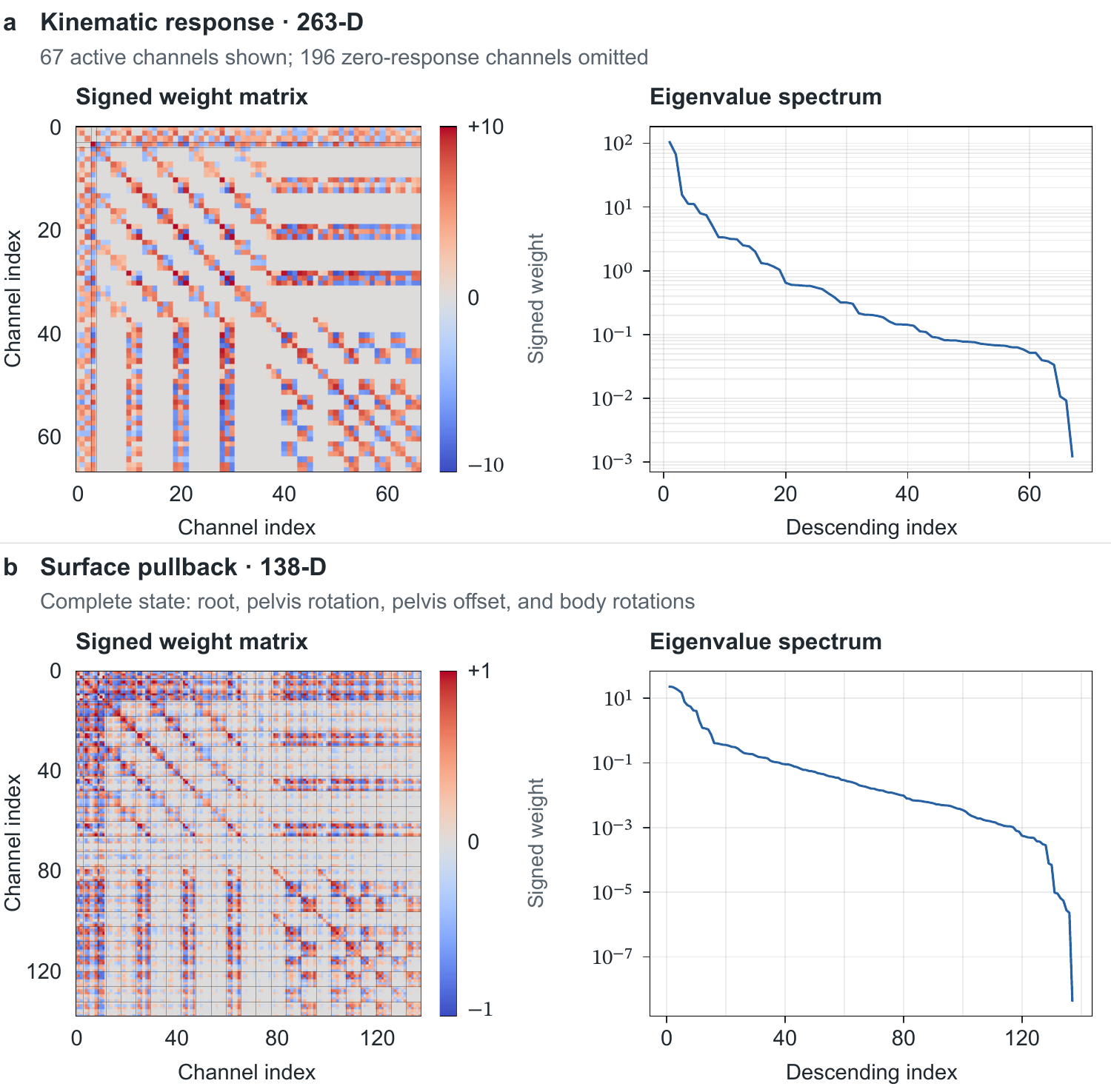}
    \caption{HumanML3D-calibrated FK matrices. Trace-normalized kinematic-response (263-D)
    and surface-pullback (138-D) matrices with descending eigenvalue spectra.
    The 263-D heatmap shows the active 67-channel block; omitted channels
    have zero weight. Each heatmap uses its own signed color scale.}
    \label{fig:fk_matrix_analysis}
\end{figure}

\begin{samepage}
The 138-D surface matrix has off-diagonal ratio 55.43\%
(\Cref{fig:fk_matrix_analysis}); 99.11\% of the off-diagonal squared energy
couples different semantic channel groups rather than coordinates within
one 6D rotation block. Its ten largest eigenvalues carry 90\% of the trace.
The two FD2 zero modes arise because, within one frame, planar root
displacement can cancel the oppositely scaled planar pelvis offset.
The identity term in \Cref{eq:app_fixed_mesh_weight} removes this gauge.
Fixing the root path gives the 135-D matrix full numerical rank.
At $\gamma=1$, $W_\gamma=(I+W_{\mathrm{FK}})/2$ has condition numbers
23.8292 and 23.0458 for FD2 and FD2-Path, respectively, and trace $d_b$.\par
\end{samepage}

\subsection{Effect of FK Supervision on the Sampling Distribution}
\label{sec:app_fk_sampling}

We compare pure nonlinear FK, pure FK-induced quadratic, and MSE supervision
in a controlled 6-D rotation experiment. The dataset contains 8,192 equally
weighted rotations distributed across eight regions. Their centers are the
identity, $+90^\circ$ rotations about the $x$, $y$, and $z$ axes,
$-90^\circ$ rotations about the same axes, and a $100\sqrt{2}^\circ$
rotation about $(1,1,0)^\top/\sqrt{2}$, in that order.
The region sizes are $(2458,1720,1311,983,655,491,328,246)$.
Each rotation is independently perturbed by left-multiplying its center
by an isotropic Gaussian rotation-vector perturbation, with per-component
standard deviations $(4,5,6,7,9,11,13,16)^\circ$, respectively.
Thus region~1 has the highest density, while all entries have the same
sampling probability. Rotations are encoded by their first two columns.
The FK map $F:\mathbb R^6\to\mathbb R^9$ orthonormalizes these columns
and returns the positions of three orthogonal unit-axis endpoints.

For $\bx_t=t\bz+(1-t)\bepsilon$, $\bepsilon\sim\mathcal N(0,I_6)$,
let $a$ be the predicted velocity, $y=\bz-\bepsilon$, and
$\widehat{\bz}=\bx_t+(1-t)a$. The three standalone losses are
\begin{equation}
\begin{aligned}
    \ell_{\mathrm{FK}}(y,a)
    &=\tfrac19\|F(\widehat{\bz})-F(\bz)\|_2^2,\\
    \ell_{\mathrm{quad}}(y,a)
    &=\tfrac16(a-y)^\top W_{\mathrm{FK}}(a-y),\qquad
    \ell_{\mathrm{MSE}}(y,a)=\tfrac16\|a-y\|_2^2.
\end{aligned}
\label{eq:app_three_sampling_losses}
\end{equation}
Here $W_{\mathrm{FK}}$ is the trace-normalized average of
$J_F(\bz)^\top J_F(\bz)$ over the dataset, as in
\Cref{eq:app_fixed_mesh_weight}; it is positive definite in this experiment.

For the stated empirical data law, the exact conditional expectations can
be written using posterior weights over the $N=8{,}192$ training rotations:
\begin{equation}
    w_i(\bx_t,t)=
    \frac{\exp\!\left(-\|\bx_t-t\bz_i\|_2^2/[2(1-t)^2]\right)}
         {\sum_{j=1}^{N}\exp\!\left(-\|\bx_t-t\bz_j\|_2^2/[2(1-t)^2]\right)}.
    \label{eq:app_rotation_posterior}
\end{equation}
At fixed $(\bx_t,t)$, the corresponding velocity target is
$y_i=(\bz_i-\bx_t)/(1-t)$. Thus the conditional targets are
$\overline y=\sum_iw_i y_i$ for MSE and the fixed quadratic, and
$\overline F=\sum_iw_i F(\bz_i)$ for the direct geometric loss.
Removing the parameter-independent conditional variance gives losses
$\|a-\overline y\|_2^2/6$,
$(a-\overline y)^\top W_{\mathrm{FK}}(a-\overline y)/6$, and
$\|F(\bx_t+(1-t)a)-\overline F\|_2^2/9$, respectively.

All three velocity networks have three hidden layers of 256 SiLU units
and share the same initialization and training batches within each seed.
We use 50,000 AdamW updates with batch size 1,024, a cosine learning-rate
schedule from $2\times10^{-4}$ to $2\times10^{-5}$, and
$t\sim\mathcal U(0.001,0.999)$. Expected training losses are evaluated
using exact conditional expectations over the empirical dataset, with
parameter-independent variance terms omitted. The predicted velocity is
integrated directly in the six-dimensional coordinate. For each of two
training seeds, we generate 8,192 samples from the same Gaussian initial
states using 512 Euler steps and decode the final rotations with $F$.

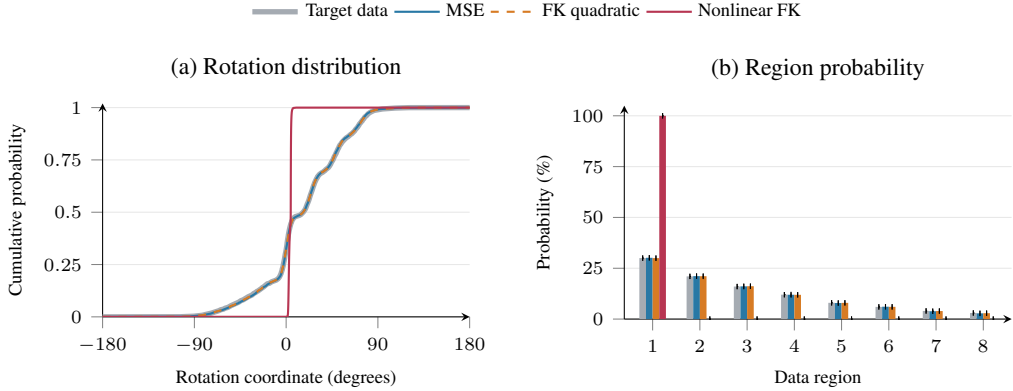
\begin{figure}[htbp]
    \centering
    \begingroup
\definecolor{fkTarget}{HTML}{A5ABB2}
\definecolor{fkMSE}{HTML}{2878A5}
\definecolor{fkQuad}{HTML}{D77B24}
\definecolor{fkDirect}{HTML}{B73552}
\pgfplotsset{fkAxis/.style={font=\scriptsize,width=.96\linewidth,height=4.4cm,
 axis lines=left,tick align=outside,ymajorgrids,grid style={gray!18},
 label style={font=\scriptsize},title style={font=\small},
 scaled ticks=false,legend style={draw=none,font=\scriptsize,legend columns=4}}}
\centering\pgfplotslegendfromname{fkSamplingLegend}\par\vspace{2mm}
\begin{minipage}[t]{0.48\linewidth}
\begin{tikzpicture}
\begin{axis}[fkAxis,title={(a) Rotation distribution},xmin=-180,xmax=180,
 ymin=0,ymax=1.02,xtick={-180,-90,0,90,180},ytick={0,0.25,0.5,0.75,1},
 xlabel={Rotation coordinate (degrees)},ylabel={Cumulative probability},
 legend to name=fkSamplingLegend]
\addplot[fkTarget,line width=2pt] coordinates {(-180.0,0.00000) (-104.0,0.00061) (-90.4,0.00317) (-82.5,0.00745) (-73.2,0.01807) (-66.1,0.02930) (-64.1,0.03394) (-61.6,0.03674) (-59.8,0.04138) (-57.1,0.04529) (-53.8,0.05420) (-46.6,0.06995) (-44.4,0.07715) (-42.1,0.08142) (-37.3,0.09570) (-33.4,0.10461) (-31.8,0.11157) (-30.0,0.11621) (-28.5,0.12207) (-27.6,0.12378) (-23.5,0.13831) (-23.2,0.14026) (-22.1,0.14307) (-19.0,0.15601) (-18.2,0.15833) (-16.5,0.16028) (-14.0,0.16797) (-9.8,0.17542) (-7.9,0.18225) (-6.8,0.18884) (-6.4,0.19189) (-5.0,0.20642) (-4.8,0.21069) (-4.2,0.21899) (-3.8,0.22778) (-3.7,0.22839) (-3.6,0.23120) (-3.0,0.24207) (-2.7,0.25146) (-2.0,0.26709) (-1.4,0.28284) (-1.3,0.28687) (-0.8,0.30066) (-0.4,0.30908) (0.0,0.31995) (0.7,0.34155) (0.9,0.34961) (1.0,0.35486) (1.5,0.37122) (2.1,0.38782) (2.2,0.39258) (3.2,0.41724) (3.7,0.42798) (4.6,0.44458) (5.1,0.44983) (5.5,0.45654) (6.2,0.46387) (7.6,0.47339) (9.0,0.47888) (13.0,0.48499) (14.5,0.48926) (17.0,0.50171) (18.5,0.51233) (19.1,0.51770) (20.9,0.54150) (21.2,0.54419) (22.3,0.55798) (23.6,0.57886) (25.7,0.60791) (25.8,0.61035) (26.9,0.62671) (27.1,0.62830) (28.1,0.64246) (29.8,0.66003) (30.5,0.66614) (30.7,0.66675) (31.3,0.67285) (32.8,0.68140) (34.6,0.68787) (35.7,0.69043) (38.2,0.69971) (38.7,0.70288) (41.4,0.71338) (42.5,0.72034) (43.7,0.73010) (44.0,0.73401) (44.8,0.74133) (45.9,0.74988) (46.8,0.76013) (47.0,0.76343) (48.6,0.77759) (48.9,0.78174) (49.1,0.78271) (50.7,0.80066) (51.7,0.80957) (52.1,0.81506) (52.9,0.82178) (53.5,0.82446) (54.9,0.83508) (57.4,0.84839) (59.0,0.85559) (60.8,0.86121) (63.4,0.87231) (64.3,0.87463) (67.1,0.88831) (67.9,0.89453) (70.2,0.90735) (70.9,0.91345) (71.3,0.91528) (71.9,0.92139) (73.2,0.92859) (73.8,0.93323) (74.5,0.93591) (74.9,0.94019) (76.3,0.94763) (76.8,0.95251) (77.5,0.95642) (78.6,0.96033) (80.5,0.97058) (82.4,0.97705) (86.1,0.98621) (94.6,0.99353) (106.7,0.99768) (120.5,0.99976) (180.0,1.00000)};
\addlegendentry{Target data}
\addplot[fkMSE,line width=.8pt] coordinates {(-180.0,0.00000) (-103.8,0.00061) (-94.0,0.00232) (-83.3,0.00702) (-73.9,0.01562) (-69.8,0.02307) (-62.9,0.03168) (-56.8,0.04364) (-53.5,0.05164) (-50.2,0.05731) (-48.9,0.06128) (-46.9,0.06506) (-46.3,0.06775) (-45.2,0.06946) (-43.6,0.07465) (-42.1,0.07733) (-38.2,0.08777) (-36.6,0.09344) (-34.3,0.09821) (-30.6,0.10962) (-27.4,0.12323) (-24.5,0.13177) (-22.0,0.14209) (-20.1,0.14685) (-18.3,0.15454) (-15.2,0.16339) (-10.2,0.17432) (-8.6,0.17926) (-7.7,0.18274) (-5.7,0.19489) (-5.0,0.20264) (-4.5,0.21039) (-2.6,0.24774) (-1.6,0.27423) (-0.8,0.30017) (-0.5,0.30768) (0.0,0.32367) (0.9,0.34833) (1.2,0.35913) (1.9,0.38037) (2.1,0.38800) (2.7,0.40411) (3.3,0.41791) (3.9,0.42938) (4.2,0.43372) (4.4,0.43823) (5.0,0.44769) (5.4,0.45349) (6.3,0.46362) (7.2,0.47003) (8.7,0.47552) (11.9,0.48376) (13.3,0.48615) (15.2,0.49261) (16.6,0.49945) (18.7,0.51392) (20.6,0.53729) (21.3,0.54413) (23.4,0.57532) (24.9,0.59961) (25.6,0.60815) (27.3,0.63446) (28.6,0.65057) (29.3,0.65735) (31.5,0.67456) (32.8,0.68182) (34.3,0.68860) (36.4,0.69427) (38.2,0.70178) (39.9,0.70721) (41.1,0.71326) (42.4,0.72168) (43.4,0.72931) (43.9,0.73468) (44.7,0.74066) (45.9,0.75183) (47.7,0.77399) (49.2,0.78949) (50.1,0.79669) (51.5,0.81073) (54.9,0.84033) (57.9,0.85425) (59.2,0.85760) (60.0,0.86163) (61.1,0.86456) (63.1,0.87250) (66.3,0.88727) (69.5,0.90601) (70.5,0.91388) (71.6,0.91943) (75.7,0.94958) (78.4,0.96576) (79.7,0.97180) (83.5,0.98297) (87.6,0.98907) (91.5,0.99249) (110.2,0.99890) (180.0,1.00000)};
\addlegendentry{MSE}
\addplot[fkQuad,dashed,line width=.8pt] coordinates {(-180.0,0.00000) (-103.1,0.00049) (-93.3,0.00244) (-82.1,0.00800) (-73.9,0.01581) (-68.1,0.02594) (-62.9,0.03259) (-55.9,0.04700) (-53.4,0.05334) (-50.9,0.05737) (-44.5,0.07483) (-38.5,0.08911) (-37.5,0.09271) (-33.7,0.10126) (-31.5,0.10846) (-29.6,0.11652) (-28.5,0.11951) (-27.1,0.12628) (-24.8,0.13239) (-22.2,0.14307) (-19.9,0.14984) (-18.1,0.15668) (-17.1,0.15857) (-15.6,0.16400) (-9.4,0.17706) (-7.7,0.18329) (-6.1,0.19305) (-5.2,0.20209) (-4.7,0.20905) (-2.6,0.25024) (-1.7,0.27356) (-0.7,0.30469) (0.9,0.34918) (1.3,0.36340) (1.8,0.37769) (2.4,0.39667) (2.7,0.40436) (3.1,0.41193) (3.4,0.41937) (4.2,0.43463) (5.4,0.45343) (6.2,0.46216) (6.9,0.46790) (8.8,0.47565) (10.3,0.48041) (13.0,0.48517) (14.7,0.49054) (16.1,0.49664) (17.6,0.50604) (19.2,0.52063) (19.7,0.52765) (20.5,0.53558) (21.7,0.54987) (22.4,0.56171) (23.4,0.57581) (24.1,0.58771) (24.6,0.59418) (24.9,0.59949) (25.5,0.60681) (27.5,0.63776) (28.8,0.65277) (29.7,0.66089) (31.2,0.67236) (32.8,0.68213) (34.3,0.68884) (36.7,0.69489) (39.4,0.70471) (41.0,0.71246) (42.2,0.71979) (45.5,0.74805) (46.2,0.75519) (47.4,0.77026) (48.5,0.78174) (50.4,0.79865) (51.9,0.81372) (54.4,0.83563) (55.7,0.84412) (57.4,0.85144) (62.8,0.86975) (66.4,0.88593) (69.8,0.90582) (70.7,0.91315) (72.7,0.92535) (75.3,0.94501) (76.1,0.94922) (77.9,0.96143) (79.4,0.96918) (83.7,0.98273) (87.8,0.98871) (90.3,0.99133) (104.9,0.99725) (111.3,0.99915) (124.6,0.99988) (180.0,1.00000)};
\addlegendentry{FK quadratic}
\addplot[fkDirect,line width=.8pt] coordinates {(-180.0,0.00000) (1.1,0.00031) (1.4,0.00269) (1.6,0.00653) (1.8,0.01636) (1.9,0.02606) (2.0,0.03882) (2.1,0.05518) (2.2,0.07562) (2.3,0.09998) (2.4,0.12628) (2.5,0.15540) (2.6,0.18799) (2.7,0.21893) (2.8,0.24768) (2.9,0.27454) (3.0,0.29834) (3.2,0.33923) (3.3,0.35620) (3.4,0.37146) (3.5,0.38495) (3.7,0.40686) (3.9,0.42322) (4.1,0.43719) (4.4,0.45581) (4.5,0.48706) (4.6,0.67621) (4.7,0.79749) (4.8,0.86139) (4.9,0.89948) (5.0,0.92334) (5.1,0.94098) (5.2,0.95380) (5.3,0.96381) (5.5,0.97699) (5.7,0.98413) (5.9,0.98871) (6.5,0.99609) (7.2,0.99829) (9.9,0.99982) (180.0,1.00000)};
\addlegendentry{Nonlinear FK}
\end{axis}\end{tikzpicture}
\end{minipage}\hfill
\begin{minipage}[t]{0.50\linewidth}
\begin{tikzpicture}
\begin{axis}[fkAxis,title={(b) Region probability},xmin=0.4,xmax=8.6,
 ymin=0,ymax=105,xtick={1,2,3,4,5,6,7,8},ytick={0,25,50,75,100},
 xlabel={Data region},ylabel={Probability (\%)},ybar=0pt,bar width=2.5pt]
\addplot[draw=none,fill=fkTarget,error bars/.cd,y dir=both,y explicit,error bar style={black,line width=.3pt},error mark options={mark size=1pt}] coordinates {(1,30.00488) +- (0,0.00000) (2,20.99609) +- (0,0.00000) (3,16.00342) +- (0,0.00000) (4,11.99951) +- (0,0.00000) (5,7.99561) +- (0,0.00000) (6,5.99365) +- (0,0.00000) (7,4.00391) +- (0,0.00000) (8,3.00293) +- (0,0.00000)};
\addplot[draw=none,fill=fkMSE,error bars/.cd,y dir=both,y explicit,error bar style={black,line width=.3pt},error mark options={mark size=1pt}] coordinates {(1,30.14526) +- (0,0.03052) (2,21.13037) +- (0,0.06104) (3,16.09497) +- (0,0.15259) (4,11.96899) +- (0,0.15259) (5,7.89185) +- (0,0.07935) (6,6.08521) +- (0,0.01831) (7,3.85742) +- (0,0.00000) (8,2.82593) +- (0,0.06714)};
\addplot[draw=none,fill=fkQuad,error bars/.cd,y dir=both,y explicit,error bar style={black,line width=.3pt},error mark options={mark size=1pt}] coordinates {(1,29.98047) +- (0,0.07324) (2,21.08765) +- (0,0.11597) (3,16.16211) +- (0,0.06104) (4,11.93237) +- (0,0.01831) (5,7.97729) +- (0,0.01831) (6,6.01807) +- (0,0.09766) (7,3.94897) +- (0,0.01831) (8,2.89307) +- (0,0.06104)};
\addplot[draw=none,fill=fkDirect,error bars/.cd,y dir=both,y explicit,error bar style={black,line width=.3pt},error mark options={mark size=1pt}] coordinates {(1,99.98779) +- (0,0.00000) (2,0.00000) +- (0,0.00000) (3,0.01221) +- (0,0.00000) (4,0.00000) +- (0,0.00000) (5,0.00000) +- (0,0.00000) (6,0.00000) +- (0,0.00000) (7,0.00000) +- (0,0.00000) (8,0.00000) +- (0,0.00000)};
\end{axis}\end{tikzpicture}\end{minipage}\endgroup
    \caption{Sampling under three standalone regression losses.
    (a) Cumulative distributions of the decoded rotation vector projected
    onto $(1,2,3)^\top/\sqrt{14}$. (b) Probability assigned to each data region, determined
    by the nearest training rotation in FK output space.
    Curves and bars show means over two training seeds; whiskers show their
    ranges. MSE and the FK-induced quadratic retain all eight regions,
    whereas nonlinear FK concentrates samples in the highest-density region.}
    \label{fig:app_fk_sampling}
\end{figure}

As shown in \Cref{fig:app_fk_sampling}, MSE and the FK-induced quadratic
assign $30.15\%$ and $29.98\%$ of samples to the highest-density region,
respectively, close to its $30.00\%$ target probability, and retain the
remaining regions. Pure nonlinear FK instead assigns $99.99\%$ to this
region in both runs.
This experiment illustrates how nonlinear FK supervision can bias sampling
toward high-density regions and reduce diversity, while the fixed quadratic
preserves the conditional-mean target.

\section{Root-Path Conditioning and Motion Representation}
\label{sec:app_representation}

This section expands the root-path conditioning in Section~\ref{sec:position_control}.
FD2-Path receives a frame-aligned planar root position and heading, encodes
them as three relative root channels, and keeps these channels clean while
diffusing the remaining body state. The root defines the commanded path;
the pelvis retains its relative translation and rotation so that body
articulation can vary around this path.

SMPL-H reconstruction requires parent-local joint rotations and a global
translation. HumanML3D-263 includes RIC joint positions and an IK-derived
$21\times6$ rotation block; its joint positions do not uniquely determine
bone twist, and its root stream retains heading rather than full pelvis
orientation. Conversion to SMPL therefore requires approximate retargeting.

For SEED generation, our 138-D state stores the heading-relative pelvis rotation and the 21
source parent-local body rotations directly. With its translation channels,
it reconstructs the 22 body joints and SMPL-H surface up to global planar
placement and numerical precision. All encoding and decoding use
neutral SMPL-H~\citep{romero2017embodied,mahmood2019amass} with
$\boldsymbol\beta=\mathbf0$ and the fixed model-mean hand pose, excluding
actor-specific shape and variable hand articulation.

The representation diagnostics below use HumanML3D source motions to
evaluate root smoothness and rotation recovery. The HumanML3D generation
benchmark in the main text instead uses the native 263-D features.

\subsection{Root--Pelvis Separation}

We separate a planar root-motion node from the articulated pelvis. The
root carries ground position and heading; the pelvis stores its vertical
motion, local displacement, and residual 3D orientation relative to that
node. This separates the user-facing trajectory from gait-cycle pelvis
sway.

\subsection{Mass-Weighted Ground-Projected Root Pivot}

Let $\bJ_{k,i}\in\mathbb R^3$, $i=0,\ldots,21$, denote the 22 SMPL-H body-joint
positions. We define an approximate whole-body pivot as
\begin{equation}
    \bq_k=\Pi_{\mathrm{ground}}\!
    \left(\sum_{i=0}^{21}w_i\bJ_{k,i}\right),
    \qquad \sum_i w_i=1,
    \label{eq:app_com}
\end{equation}
where $\Pi_{\mathrm{ground}}(x,y,z)=(x,0,z)$. The weights are derived from
the male-column adjusted Zatsiorsky--Seluyanov body-segment parameters of
\citet{deleva1996adjustments} and are fixed across sequences. For a segment of mass $m$ whose center is at
fraction $r$ from proximal landmark $P$ to distal landmark $D$, the lever
rule assigns $m(1-r)$ to $P$ and $mr$ to $D$.

We map the landmarks of \citet{deleva1996adjustments} to SMPL as follows:
CERV$\approx$Neck, XYPH$\approx$Spine2,
OMPH$\approx\operatorname{mid}(\text{Pelvis},\text{Spine1})$, and
MIDH$=\operatorname{mid}(\text{L\_Hip},\text{R\_Hip})$. The OMPH and MIDH
shares are each split equally between the two joints of their respective
midpoint mappings. Because the 22-joint
skeleton has no separate head-vertex (crown) landmark or hand joint, head
mass is assigned to Head
and hand mass to Wrist; heel is approximated by Ankle. \Cref{tab:com_weights}
lists every resulting node assignment, and \Cref{fig:com_mapping} visualizes
the anatomical-to-SMPL mapping and resulting ground projection.

\begin{figure*}[t]
    \centering
    \includegraphics[width=\textwidth]{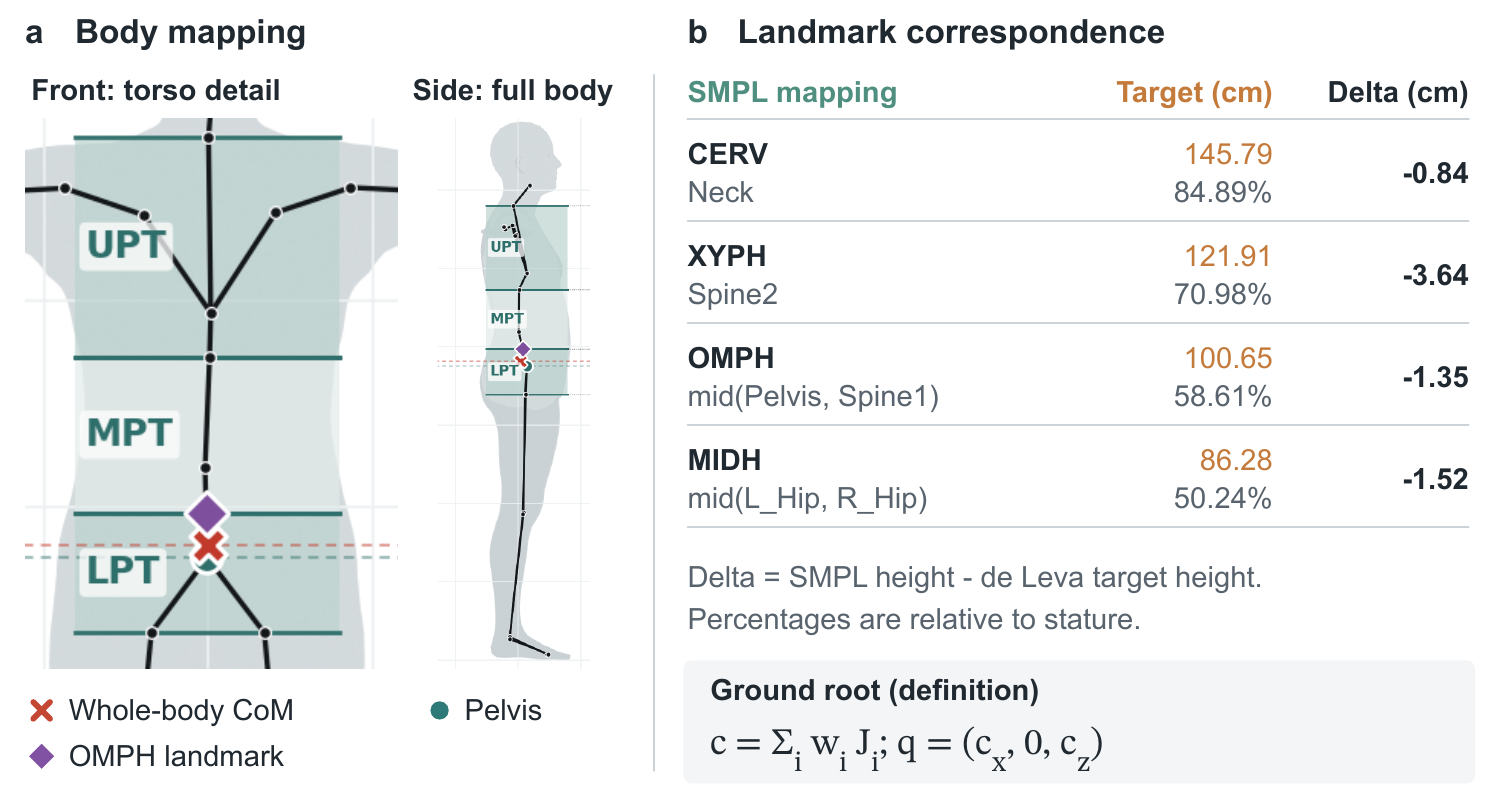}
    \caption{Biomechanical center-of-mass mapping to the neutral
    SMPL-H body. A front torso detail and a full-body side view locate
    the anatomical landmarks and the mass-weighted whole-body center of mass.
    The table compares the landmark heights from
    \citet{deleva1996adjustments} with their SMPL counterparts.
    The ground root is the planar projection of the resulting center of mass.}
    \label{fig:com_mapping}
\end{figure*}

\begin{table}[t]
\centering
\small
\caption{SMPL-H node weights for the ground-projected root pivot.
Weights follow the segment masses and center locations in
\citet{deleva1996adjustments}, using the stated landmark mapping and
normalization to unit sum.}
\label{tab:com_weights}
\setlength{\tabcolsep}{6pt}
\begin{tabular}{l r l r}
\toprule
SMPL joint & Weight & SMPL joint & Weight \\
\midrule
Pelvis      & 0.05846 & Neck       & 0.07875 \\
L\_Hip      & 0.11777 & L\_Collar  & 0.00000 \\
R\_Hip      & 0.11777 & R\_Collar  & 0.00000 \\
Spine1      & 0.05846 & Head       & 0.06940 \\
L\_Knee     & 0.08225 & L\_Shoulder& 0.01146 \\
R\_Knee     & 0.08225 & R\_Shoulder& 0.01146 \\
Spine2      & 0.17064 & L\_Elbow   & 0.02443 \\
L\_Ankle    & 0.02668 & R\_Elbow   & 0.02443 \\
R\_Ankle    & 0.02668 & L\_Wrist   & 0.01351 \\
Spine3      & 0.00000 & R\_Wrist   & 0.01351 \\
L\_Foot     & 0.00605 & R\_Foot    & 0.00605 \\
\bottomrule
\end{tabular}
\end{table}

A pelvis projection can oscillate laterally during walking. The mass-weighted
construction is designed to let opposing limb contributions cancel before
ground projection, while the difference between pelvis and pivot remains
explicitly represented. We verify this effect on the complete HumanML3D
training split in \Cref{tab:root_smoothness}. Of 21,466 listed training entries,
21,462 are valid for this computation. Relative to the pelvis projection, the
mass-weighted center-of-mass projection reduces planar jerk by 20.3\% and curvature by
9.6\%.

\begin{table}[t]
\centering
\small
\caption{Smoothness of the global root definition. Statistics are
computed over the complete HumanML3D training split at 30 fps, with each
motion weighted equally. Jerk is the mean norm
of the third temporal derivative of planar position; curvature is evaluated
at frames whose planar speed exceeds $0.05\,\mathrm{m/s}$.}
\label{tab:root_smoothness}
\setlength{\tabcolsep}{7pt}
\begin{tabular}{lcc}
\toprule
Root definition & Jerk $\downarrow$ ($\mathrm{m/s^3}$) & Curvature $\downarrow$ ($\mathrm{m^{-1}}$) \\
\midrule
Pelvis projection (263-style) & 33.450 & 32.632 \\
Mass-weighted center-of-mass projection (ours) & \textbf{26.650} & \textbf{29.504} \\
\bottomrule
\end{tabular}
\end{table}

The heading-local
root displacement and pelvis offset are
\begin{align}
    \bd_k &=
    \Pi_{xz}\!\left(
       R_y(\psi_k)^\top(\bq_k-\bq_{k-1})
    \right)\in\mathbb R^2, \label{eq:app_root_displacement}\\
    \bo_k &=
       R_y(\psi_k)^\top(\bJ_{k,\mathrm{Pelvis}}-\bq_k)
       \in\mathbb R^3 .
    \label{eq:app_root_features}
\end{align}

\subsection{Heading from Hips and Shoulders}

Pelvis orientation contains local body motion and can fluctuate even when
the intended travel direction is stable. We instead estimate a skeletal
lateral direction by averaging the hip and shoulder lines. On frames with
$\mathbf e_y\times\ba_k\ne\mathbf0$, define
\begin{align}
    \ba_k &=
       (\bJ_{k,\mathrm{RHip}}-\bJ_{k,\mathrm{LHip}})
       +(\bJ_{k,\mathrm{RShoulder}}-\bJ_{k,\mathrm{LShoulder}}),\\
    \mathbf f_k &=\frac{\mathbf e_y\times\ba_k}
                  {\max(\|\mathbf e_y\times\ba_k\|_2,\varepsilon)},\qquad
    \psi_k=\operatorname{atan2}(f_{k,x},f_{k,z}),\\
    \Delta\psi_k&=
       \operatorname{wrap}(\psi_k-\psi_{k-1}).
    \label{eq:app_heading}
\end{align}
The heading uses both landmark pairs without temporal smoothing, and the
representation stores its wrapped increment. On the nondegenerate domain
specified above, the relative motion features are invariant to global yaw.
The $\varepsilon$ guard stabilizes normalization.

\subsection{138-D Layout and Reconstruction}

The full per-frame vector is
\begin{equation}
    \mathbf m_k=
    [
      \underbrace{\Delta\psi_k}_{1},
      \underbrace{\bd_k}_{2},
      \underbrace{\br_k}_{6},
      \underbrace{\bo_k}_{3},
      \underbrace{\rho_6(R_{k,1}^{\mathrm{local}}),\ldots,
                  \rho_6(R_{k,21}^{\mathrm{local}})}_{21\times6}
    ],
    \label{eq:app_138}
\end{equation}
where $\br_k=\rho_6(R_y(\psi_k)^\top R_{k,\mathrm{Pelvis}})$ and $\rho_6$ is the continuous two-column rotation representation of
\citet{zhou2019continuity}. For a rotation matrix $R$, $\rho_6(R)$ contains
its first two columns,
$[R_{00},R_{10},R_{20},R_{01},R_{11},R_{21}]$; $\rho_6^{-1}$ applies the
standard Gram--Schmidt reconstruction. Root quantities are heading-relative, and the 21
body rotations are parent-local.

Starting from planar position $\bq_0$ and heading $\psi_0$,
decoding for $k\geq1$ uses
\begin{align}
    \psi_k &= \psi_{k-1}+\Delta\psi_k,\\
    \bq_k &= \bq_{k-1}
      +R_y(\psi_k)[d_{k,x},0,d_{k,z}]^\top,\\
    \bJ_{k,\mathrm{Pelvis}} &=
      \bq_k+R_y(\psi_k)\bo_k,\\
    R_{k,\mathrm{Pelvis}} &=
      R_y(\psi_k)\rho_6^{-1}(\br_k).
    \label{eq:app_decode}
\end{align}
At $k=0$, the stored pelvis residual and body rotations are decoded directly
under $(\bq_0,\psi_0)$. With valid encoded rotations and the same initial pose,
this reconstruction recovers the original positions and rotations; heading
is recovered modulo $2\pi$. The parent-local body rotations then recover all
remaining joint transforms through forward kinematics. The neutral-shape
SMPL-H pose corrective blend shapes and LBS use those transforms to produce
the final corresponding surface vertices. Changing $(\bq_0,\psi_0)$ places
the same relative joints and mesh at a new ground location and orientation.

\subsection{Path Encoding}
\label{sec:app_path_encoding}

FD2-Path converts a complete per-frame planar pose sequence
$(x_k,z_k,\psi_k)$ through the encoder of its motion state. In the 138-D
state, $\bc_k=[\Delta\psi_k,\bd_k]$ uses the wrapped heading change from
$k-1$ to $k$ and displacement expressed in heading $\psi_k$
(\Cref{eq:app_root_displacement,eq:app_heading}). The native 263-D
features instead store the heading rate and heading-local displacement
leaving frame $k$.

The three encoded root channels are standardized with the same fixed
channel statistics used during training and concatenated with the noisy
body channels. Only the remaining 135 or 260 channels are predicted and
supervised, as in \Cref{eq:position_condition}. During sampling, the supplied
root channels remain fixed; decoding combines them with the generated
pelvis and body state. Tracking is evaluated from the decoded motion using
the conventions in Appendix~\ref{sec:app_path_conventions}.

\subsection{Source-Rotation Fidelity}
\label{sec:results_rotation}

We encode and decode 26,846 HumanML3D source motions and compare the
recovered rotations with the source local rotations. The reported errors
are frame-weighted. Two motions are omitted from the 263-D result because
its SVD-based reconstruction does not converge. \Cref{tab:rotation_fidelity}
reports the mean geodesic error. This round trip measures preservation of
the source rotations through encoding and decoding.

\begin{table}[t]
\centering
\small
\caption{Source-rotation round trip. Frame-weighted mean geodesic error over the HumanML3D source-motion
collection (26,846 motions; two SVD failures excluded for 263-D).}
\label{tab:rotation_fidelity}
\setlength{\tabcolsep}{7pt}
\begin{tabular}{l c c c}
\toprule
Representation & Dim. & Pelvis error ($^\circ$)$\downarrow$ & Body-local error ($^\circ$)$\downarrow$ \\
\midrule
HumanML3D~\citep{guo2022generating} & 263 & $1.635$ & $30.680$ \\
MotionStreamer~\citep{xiao2025motionstreamer} & 272 & $<10^{-5}$ & $<10^{-5}$ \\
FloodDiffusion~2 (ours, 138-D) & 138 & $<10^{-5}$ & $<10^{-5}$ \\
\bottomrule
\end{tabular}
\end{table}

\section{Implementation and Evaluation Details}
\label{sec:app_experimental_details}

This section collects the training configuration and measurement protocols
for the experiments in Section~\ref{sec:experiments}. The timing and FLOP
measurements describe denoiser computation; the benchmark-specific
evaluation protocols define motion quality, path tracking, and foot skate.

\subsection{Training Configuration}
\label{sec:app_training_config}

\paragraph{\normalfont Shared architecture and optimization.}
FD2 uses a 114M-parameter transformer with 8 layers, width 1024, FFN width
2048, and 8 heads. Frozen UMT5-XXL~\citep{chung2023unimax} provides
frame-aligned text features. We use $n_s=30$, $K=8$, and bf16 training.
AdamW uses learning rate $2\times10^{-4}$, $(\beta_1,\beta_2)=(0.9,0.99)$,
cosine decay, and a batch of 128 source sequences per GPU.
The full model uses $\gamma=1$ in \Cref{eq:mesh_quadratic}; the ablation
without FK uses $\gamma=0$.

\begin{table}[htbp]
\centering
\small
\caption{Dataset-specific checkpoint and inference settings. The number of
predicted channels excludes the three clean root channels for FD2-Path.}
\label{tab:app_training_settings}
\setlength{\tabcolsep}{5pt}
\begin{tabular}{l l c c c}
\toprule
Dataset & Variant & Predicted channels & Checkpoint step & CFG scale \\
\midrule
HumanML3D & FD2 & 263 & 60K & 4 \\
HumanML3D & FD2-Path & 260 & 55K & 3 \\
SEED & FD2 & 138 & 300K & 2 \\
SEED & FD2-Path & 135 & 300K & 2 \\
\bottomrule
\end{tabular}
\end{table}

FD2 directly diffuses native 263-D features on HumanML3D and the 138-D
rotation-direct state on SEED, using a common neutral-body convention.
The latter decodes to SMPL-H as described in
Appendix~\ref{sec:app_representation}. The full-training calibration
diagnostics use 21,466 HumanML3D entries for the 138-D state and 23,384
entries for the 263-D features, the latter including additional HumanAct
data. These are the calibration sets for the matrix diagnostics in
\Cref{tab:mesh_matrix_diagnostics}. The common construction and
branch-specific normalization are given in
Appendix~\ref{sec:app_fk_implementations}.

\subsection{Training-Efficiency Measurements}
\label{sec:app_training_efficiency}

For the training-efficiency measurements in \Cref{tab:train_eff}, the $K=1$ baselines use four GPUs for 512 windows per step and eight GPUs for 1,024 windows, with 128 source sequences per GPU. Each NVIDIA H200 has 141\,GB HBM3e memory; reported memory sums peak usage across GPUs. Packing uses the per-GPU batch sizes listed in the table, which differ from the default training batch.

\subsection{KV-Cache Inference Measurements}
\label{sec:app_kv_benchmark}

\paragraph{\normalfont Timing protocol.}
\Cref{tab:kv_cache} compares uncached Full attention with KV-cached
Partial Attention under the following common measurement protocol.
We benchmark the 114M-parameter backbone (8 layers, width 1024, FFN width 2048, 8 heads, and 138-D input/output) on an NVIDIA RTX 4090 using PyTorch 2.9.1 and CUDA 12.8. Both paths use CUDA Graph replay and PyTorch SDPA with automatic backend selection. Each call processes one denoiser branch without CFG at batch 1, with 30 active frames and 0, 90, 450, or 4500 history frames. Text features are supplied as GPU-resident $16\times4096$ tensors; text projection and cross-attention are included in the measured backbone. Projection and feed-forward GEMMs use bf16 except for the time/output projections. Stored K/V uses bf16; normalization and residual arithmetic use FP32, with TF32 enabled.

After 8 warm-up calls, we measure 9 groups of 20 replays with CUDA events and report the median group-mean latency. Latency measures steady-state backbone execution with cached text features; startup, text encoding, motion decoding, and rendering are excluded.

Cached execution processes the 30 active frames and the newly finalized history frame in one forward pass, including causal re-encoding and K/V writes (31 rows, or 30 with empty history). Q/K/V projections and feed-forward layers process the rows jointly. Self-attention evaluates two query groups: the finalized frame attends to its inclusive causal prefix, and active frames attend to the complete visible prefix. Concatenating their outputs implements the Partial Attention mask. History K/V uses $2\times8\times H\times1024\times2$ bytes at history length $H$, or 140.625\,MiB at $H=4500$.

FLOPs count two operations per multiply--add in all linear and input-convolution projections and in the attention products $QK^\top$ and $AV$. Text and time projections, all eight transformer blocks, the output projection, and the newly finalized frame are included. Elementwise operations and memory transfers are excluded from the FLOP count.

\subsection{HumanML3D Evaluation}
\label{sec:app_humanml_eval}

We use the standard HumanML3D split and released text-to-motion
evaluator~\citep{guo2022generating}, with FD1's preprocessing and duration
handling. \Cref{tab:quantitative_eval} reports the mean and 95\% confidence
interval for repeated evaluations. The HumanML3D generation models operate
on the native 263-D representation; the 138-D round-trip and root-smoothness
diagnostics in Appendix~\ref{sec:app_representation} evaluate representation
properties separately.

\subsection{Path-Control Evaluation Conventions}
\label{sec:app_path_conventions}

For FD2-Path, Position in \Cref{tab:seed_results} is the
distance from the mass-weighted center-of-mass projection recomputed from the generated
body joints using \Cref{eq:app_com} to the commanded root trajectory.
Both are measured in the decoded neutral SMPL-H coordinate system, before
the SOMA conversion and its inverse scale adjustment. Distances are pooled
over conditioned frames, computed in meters, and reported in centimeters.
For the 917-case evaluation, the mean and P95 use 183,130 conditioned frames.

On HumanML3D, GMD and DART target planar pelvis positions; OmniControl targets
3D pelvis positions. FD2-Path conditions on planar root position and heading.
Root and 3D in \Cref{tab:humanml_path} measure the mean xz-plane
control-point error and 3D pelvis error relative to ground truth, respectively,
pooled over frames. Skate follows the OmniControl convention at 20 fps:
a foot is in contact when its height is below 5 cm in both adjacent frames,
and skates when both its instantaneous planar speed and its five-frame
average exceed 0.5 m/s. A frame counts as skating if either foot skates;
we report the per-sequence fraction averaged over sequences.

\subsection{SEED Evaluation}
\label{sec:app_seed_eval}

We use the public Kimodo training partition and category-held-out Content
test split of SEED, with the official v1.1 evaluator~\citep{kimodo2026benchmark}.
Path-conditioned methods receive dense paths for the same 917 cases in
\Cref{tab:seed_results}.

We decode the 138-D output into the 22 SMPL-H body-joint rotations and the
world-space pelvis trajectory. The rotations are mapped to SOMA77 using
the inverse joint correspondence of the SEED training-data conversion;
the neck rotation is assigned to Neck1, and unmapped joints retain identity
rotations. Pelvis translations are divided by the training-conversion
scale factor 0.915609. The official SOMA skeleton then produces 77-joint
positions by forward kinematics at 30 fps. These positions are passed to
the released TMR-SOMA-RP-v1 motion encoder, which applies its standard
joint selection, canonicalization, and normalization. R@3 and FID use
the official embedding and aggregation pipeline.

For generated motions, SEED Skate uses the same contact thresholds for
all methods: 0.10 m in height and 0.15 m/s in speed, implemented by the
official motion conversion routine. We report the mean 3D foot speed
over contact frames in cm/s, averaged over sequences. The official metric
returns zero for a sequence with no contact frames. Parenthesized
Kimodo and ARDY scores use the official models' predicted contact labels
on the same generated motions.

\end{document}